%% file: main.tex
\documentclass{article}

 \usepackage[preprint]{neurips_2026}

\usepackage[utf8]{inputenc} 
\usepackage[T1]{fontenc}    
\usepackage{hyperref}       
\usepackage{url}            
\usepackage{booktabs}       
\usepackage{amsfonts}       
\usepackage{nicefrac}       
\usepackage{microtype}      
\usepackage{xcolor}         
\usepackage{amsmath}
\usepackage{amssymb}
\usepackage{amsthm}
\usepackage{wrapfig}
\makeatletter
\@ifundefined{theorem}{\newtheorem{theorem}{Theorem}}{}
\@ifundefined{lemma}{\newtheorem{lemma}[theorem]{Lemma}}{}
\@ifundefined{assumption}{\newtheorem{assumption}[theorem]{Assumption}}{}
\@ifundefined{definition}{}{}
\@ifundefined{proposition}{}{}
\@ifundefined{remark}{}{}
\makeatother

\usepackage{graphicx}
\usepackage{caption}  \usepackage{subcaption}
\usepackage[capitalize,noabbrev]{cleveref}
\usepackage{algorithmic}
\usepackage{algorithm}
\input{math_commands.tex}

\crefname{assumption}{Assumption}{Assumptions}
\Crefname{assumption}{Assumption}{Assumptions}
\crefname{lemma}{Lemma}{Lemmas}
\Crefname{lemma}{Lemma}{Lemmas}
\crefname{theorem}{Theorem}{Theorems}
\Crefname{theorem}{Theorem}{Theorems}
\crefname{definition}{Definition}{Definitions}
\Crefname{definition}{Definition}{Definitions}
\crefname{proposition}{Proposition}{Propositions}
\Crefname{proposition}{Proposition}{Propositions}
\crefname{remark}{Remark}{Remarks}
\Crefname{remark}{Remark}{Remarks}
\crefname{corollary}{Corollary}{Corollaries}
\Crefname{corollary}{Corollary}{Corollaries}

\title{Goal-driven Bayesian Optimal Experimental Design for Robust Decision-Making Under Model Uncertainty}

\author{%
  Jinwoo Go$^{1}$,
  Xiaoning Qian$^{1,2,3}$,
  Byung-Jun Yoon$^{1,2}$\\[0.5em]
  $^{1}$Computing \& Data Sciences, Brookhaven National Laboratory, Upton, NY\\
  $^{2}$Department of Electrical \& Computer Engineering, Texas A\&M University, College Station, TX\\
  $^{3}$Department of Computer Science \& Engineering, Texas A\&M University, College Station, TX\\[0.5em]
  \texttt{\{jgo,xqian1,byoon\}@bnl.gov}
}

\begin{document}

\maketitle

\begin{abstract}
Bayesian optimal experimental design (BOED) selects experiments to maximize information gain about model parameters. However, in decision-critical settings, reducing parameter uncertainty does not necessarily improve downstream decisions, as only specific parameter directions relevant to the objective truly matter. We propose GoBOED, a goal-driven BOED framework that directly optimizes experimental designs for a specified decision-making objective. GoBOED combines an amortized variational posterior surrogate with a differentiable convex decision layer, enabling gradient-based design optimization that is fully decision-focused. We theoretically show that GoBOED gradients are insensitive to parameter directions irrelevant to the decision objective, providing a formal justification for why goal-driven design achieves equivalent decision quality over a wider set of experimental designs than information-gain maximization. Empirically, across source localization, epidemic management, and pharmacokinetic control, GoBOED identifies designs that better align with downstream decision objectives and reveals that near-optimal design windows are substantially wider than those predicted by goal-agnostic BOED approaches.
\end{abstract}

\section{Introduction}
When experiments are expensive, time-consuming, or carry safety risks, it is critical to extract as much useful information as possible from each experiment. \textbf{Bayesian optimal experimental design (BOED)} provides a principled framework for this: it selects experiments that maximally reduce uncertainty in model parameters, as measured by the expected change between prior and posterior distributions  \citep{chaloner1995bayesian, lindley1956measure}.
BOED has been successfully applied across diverse fields including psychology \citep{bach2023experiment} and geophysics \citep{strutz2024variational}. A detailed discussion of related work on BOED is provided in \Cref{app:related_work}.

In parallel, many high-stakes real-world systems require robust decision-making under model uncertainty. For instance, optimal control has been developed for epidemic management \citep{ma2023optimal} and pharmacokinetic (PK) dose optimization \citep{zou2020application}, where model parameters are inherently uncertain. Open-loop control fails under such uncertainty \citep{vitkova2023robust}, and robust control methods \citep{nemirovski2012robust}
 hedge against this uncertainty rather than actively reducing it through experimentation.

\begin{figure*}[t]
    \centering
    \includegraphics[width=0.8\textwidth]{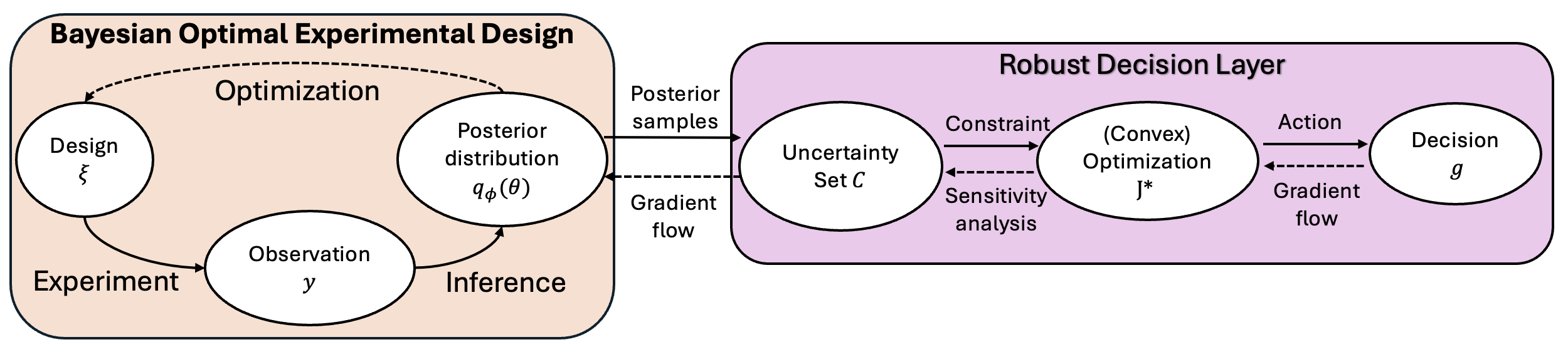}\vspace{-1mm}
    \caption{GoBOED couples Bayesian optimal experimental design (BOED) with a robust decision layer. In the BOED component (left), a design $\xi$ is applied to an experiment, producing observation $y$, which is used by an amortized inference network to update the posterior distribution $q_\phi(\btheta)$. Posterior samples are then passed to the robust decision layer (right), where they define an uncertainty set $\mathcal{C}$ that constrains a convex optimization problem $J^*(y,\xi)$, yielding a decision $\vect{g}$. Sensitivity analysis back-propagates through the convex optimization to identify which directions of posterior uncertainty most affect the decision, and this signal guides the optimization of the design $\xi$. Dashed arrows indicate gradient flow during design optimization through the posterior surrogate and decision layer. 
    }\vspace{-5mm}
    \label{fig:diag}
\end{figure*}

Despite progress in both fields, a key gap remains between BOED and robust decision-making. \textbf{Standard BOED treats all uncertainty equally and aims to design the experiment that reduces the total uncertainty, without regard to which uncertainty actually matters for the decision-making task at hand}. Reducing uncertainty that does not enhance decision-making potentially wastes valuable experimental resources. While some goal-oriented formulations exist (see \Cref{app:related_work}), they do not address decision-making; the closest work \citep{anderson2023experiment} is tied to Gaussian process-specific structure. 

To fill this critical gap, we propose Goal-driven BOED (GoBOED), a framework that couples BOED with robust decision-making to reduce uncertainty in directions that affect downstream decisions most. GoBOED differs from existing decision- and prediction-aware BOED in three concrete ways. First, it preserves the decision as an explicit constrained convex program with plug-in risk functionals, rather than a learned end-to-end policy~\citep{huang2024amortized}. Second, it uses a single amortized variational surrogate trained once offline, avoiding per-design MCMC outer loops~\citep{zhong2024goal} or GP-specific inference models~\citep{anderson2023experiment}. Finally, it directly differentiates the robust optimal value of the downstream control problem rather than an information-theoretic surrogate in parameter or prediction space~\citep{pmlr-v206-bickfordsmith23a, neiswanger2021bayesian}, and we theoretically show that this design 
gradient is insensitive to task-irrelevant parameter directions.

The overall workflow is illustrated in~\Cref{fig:diag}. We study a single-step instantiation, and validate GoBOED on source localization, epidemic management, and pharmacokinetic (PK) control.

Our main contributions are:

\squishlist\vspace{-2mm}
    \item We propose GoBOED, a framework that integrates BOED with robust
    decision-making through a differentiable convex decision layer, enabling experimental designs that directly optimize downstream decision quality.

    \item We provide a theoretical analysis showing that, under a task-relevant linear subspace structure, the GoBOED gradient is insensitive to parameter directions irrelevant to the decision objective
    (\cref{thm:null-space-cancellation}).

    \item We validate GoBOED on source localization, epidemic control, and pharmacokinetic control, showing that decision-focused designs reveal broader near-optimal windows than standard information-gain maximization.
\squishend





\section{Background}

In this section, we present a brief overview of BOED and the goal-driven formulation needed for GoBOED. For clarity, mathematical notations and symbols in this paper are shown in Table~\ref{tab:notation} of Appendix~\ref{app:notation} with additional explanations.

\subsection{Bayesian optimal experimental design (BOED)}
\label{sec:back-boed}
BOED is an information-theoretic approach to the problem of identifying which experiments are most informative. It consists of a prior assumption about unknown model parameters $\btheta \in \Theta$, design variables $\xi \in \Xi$, a forward model $f:\Theta\times\Xi \rightarrow Y$, and observations $y \in Y$ modeled by $y = f(\btheta, \xi) + \epsilon$, where $\epsilon$ denotes noise, e.g., drawn from a Gaussian distribution of known mean and covariance. 

Under this setting, we evaluate experimental designs by computing the expected information gain~(EIG). This involves calculating the KL divergence between the posterior and prior distributions, and taking the expectation of this value with respect to the marginal likelihood. Using Bayes' rule, the KL divergence form can be reformulated to estimate EIG by sample average approximation~(SAA) as 
\begin{equation}
\mathrm{EIG}(\xi) := \mathbb{E}_{p(\btheta)p(y\mid\btheta,\xi)} \left[ \log \frac{p(y\mid\btheta,\xi)}{p(y\mid\xi)} \right], \label{eq:eig}
\end{equation}
with the most informative experiment design computed by solving $\xi^* = \arg\max_\xi \mathrm{EIG}(\xi)$.  

EIG can be estimated via flexible neural approaches \citep{dong2025variational, iollobayesian, kleinegesse2020bayesian}. Here we use variational inference
\citep{foster2019variational, foster2020unified} for its computational efficiency and asymptotic soundness in scientific settings.

\subsection{Goal-driven Problem Formulation}
\label{sec:problem-formulation}

While standard BOED focuses exclusively on parameter estimation, in many real-world applications the inferred parameters are often an intermediate step toward a downstream decision $\vect{g} \in \mathcal{G}$ for minimizing a control cost $J(\vect{g})$. The uncertain parameters do not enter the cost directly, but they govern the physical or safety boundaries of the system through $m$ inequality constraints $c_j(\vect{g}, \btheta) \le 0$ for $j=1,\ldots,m$. Because the true parameters remain uncertain even after observing data $y$ under design $\xi$, the decision must be robust to the updated posterior.

We formalize this by defining the optimal value function $J^\ast(y, \xi)$ of the posterior decision problem as:
\begin{equation}
\begin{split}
J^\ast(y, \xi) = \min_{\vect{g} \in \mathcal{G}} J(\vect{g}) \ \
\mathrm{s.t.} \ \rho_{\btheta \sim p(\btheta\mid y,\xi)} [c_j(\vect{g}, \btheta) ]\le 0, \quad j=1,\ldots,m
\end{split}
\label{eq:opt-theta}
\end{equation}
where $\rho$ is a chosen risk functional — such as expectation, chance constraints, or Conditional Value at Risk~(CVaR) — that maps the posterior uncertainty of each constraint to a deterministic scalar. 

The overarching GoBOED objective then seeks the experimental design $\xi^*$ that minimizes the expected robust decision cost, marginalized over all possible observations before they are collected:
\begin{equation}
\xi^* = \arg\min_{\xi\in\Xi} \; \mathbb{E}_{y \sim p(y\mid \xi)} \! \big[ J^\ast(y, \xi) \big].
\label{eq:opt-design}
\end{equation}
Unlike standard BOED, which maximizes \cref{eq:eig}, \cref{eq:opt-design} is strictly decision-focused. It prioritizes experiments that reduce the eventual robust decision loss the most while accounting for downstream feasibility.

\section{Methods}

We develop an integrated framework, GoBOED, for goal-driven experimental design by bringing together BOED and robust decision-making under model parameter uncertainty through convex optimization. Our starting point is the general goal-driven BOED objective in \cref{eq:opt-theta}, which seeks an experimental design $\xi^*$ that minimizes the expected robust decision cost. In the remainder of this section, we instantiate this objective with a convex robust decision layer, and the overall workflow is illustrated in \cref{fig:diag}.

\subsection{Overview and scope}
\label{sec:overview}

We instantiate GoBOED in a single-step, non-adaptive design setting, in which a design $\xi$ is selected before any data are acquired and a single observation $y$ is used to inform the downstream decision. Sequential and adaptive extensions are discussed in \cref{sec:conclusion}. Because our goal is the quality of the induced decision rather than parameter estimation, the pipeline is organized around three coupled components that together define the end-to-end map from $\xi$ to the robust optimal value $J^\ast(y,\xi)$.

We first introduce an amortized variational posterior surrogate $q_\phi(\btheta \mid y,\xi)$ (\cref{sec:posterior}), which turns a design $\xi$ and its observation $y$ into a posterior distribution over $\btheta$. The posterior samples are then fed into a robust decision layer (\cref{sec:3-rdd}), where posterior uncertainty over $\btheta$ is translated into a deterministic convex optimization problem which calculates the robust optimal value $J^\ast(y,\xi)$. Finally, we differentiate $J^\ast(y,\xi)$ with respect to $\xi$ (\cref{sec:decision-layer}) using the reparameterization trick through the posterior surrogate and implicit differentiation through the convex decision layer.

These choices are the minimal set that makes the full pipeline end-to-end differentiable, and they allow us to decouple offline posterior learning from online design optimization in the two-stage procedure of \cref{sec:two-stage}.

\subsection{Posterior surrogate}
\label{sec:posterior}

We approximate the true posterior $p(\btheta \mid y, \xi)$ with a variational distribution $q_\phi(\btheta \mid y, \xi)$, where $\phi$ denotes the variational parameters. This variational approximation serves two roles: it provides a tractable posterior surrogate for downstream decision-making, and it acts as the proposal for self-normalized importance sampling (IS) used to correct for the bias of $q_\phi$ when evaluating downstream risk functionals. The variational parameters are optimized by maximizing the ELBO: 
\begin{equation}
\mathcal{L}_\text{VI}(\phi; y,\xi) = \mathbb{E}_{q_{\phi}(\btheta \mid y, \xi)}
\left[ \log \frac{p(y \mid \btheta, \xi)p(\btheta)}{q_{\phi}(\btheta \mid y, \xi)}
\right]. \label{eq:SVI}
\end{equation}
We adopt the ELBO rather than forward-KL / conditional log-likelihood objectives~\citep{papamakarios2019sequential}  because the explicit likelihood makes it directly tractable and its mode-seeking behavior yields better-behaved IS weights (see Appendix~\ref{app:is-estimator}).


Once the variational posterior $q_\phi$ is obtained, we use it as the proposal distribution for IS to estimate expectations under the true posterior. For any $f(\btheta)$,
$\mathbb{E}_{p(\btheta \mid y, \xi)}[f(\btheta)] \approx \sum_{i=1}^{N} \tilde{w}(\btheta_{i})\, f(\btheta_{i}), \label{eq:is-theta}$
where $\btheta_{i}\sim q_\phi(\btheta\mid y,\xi)$ and the weights $\tilde{w}(\btheta_i)$ are defined in Appendix~\ref{app:is-estimator}. This estimator is reused throughout \cref{sec:3-rdd}.



\subsection{Robust decision layer}
\label{sec:3-rdd}

Recall from \cref{sec:problem-formulation} that the downstream robust control problem \cref{eq:opt-theta} minimizes the cost $J(\vect g)$ subject to risk-aware constraints evaluated under posterior samples from $p(\btheta \mid y, \xi)$. For the remainder of the paper, we assume $J(\vect g)$ is convex in $\vect g$. This specialization yields a differentiable downstream optimization problem that can be embedded into the GoBOED design loop. Non-convex control objectives generally yield discontinuous solution maps $\vect{g}^*$, which renders Karush–Kuhn–Tucker (KKT)-based Jacobians undefined, and extending GoBOED to such settings is an interesting direction for future work. To maintain the end-to-end differentiability required for GoBOED, we restrict our scope to convex formulations.

Given the posterior uncertainty in $\btheta$, each constraint value $c_j(\vect g,\btheta)$ is a random variable. There is no single universally best notion of robustness, so we consider multiple risk-aware formulations that translate this uncertainty into a deterministic optimization problem. In what follows, we instantiate the general risk functional $\rho$ with three standard choices: mean-based constraints (risk-neutral), chance constraints (probability-level), and CVaR (tail-expectation), providing a spectrum from risk-neutral to tail-robust behavior.

\subsubsection{Mean risk functional}

We approximate the posterior mean $\bar{\btheta}$ using IS with $f(\btheta) = \btheta$, yielding $\bar{\btheta} = \sum_{i=1}^N \tilde{w}_i \btheta_i$, and use this mean directly in the constraint. The resulting formulation, denoted $J^\ast_{\text{mean}}(y, \xi)$, is 
\begin{equation}
J^\ast_{\text{mean}}(y, \xi) = \min_{\vect g} J(\vect g)
\ \ \text{s.t. } c(\vect g, \bar{\btheta}) \le 0,
\label{eq:mean-constraint-theta}
\end{equation}
which can be viewed as a deterministic approximation obtained by enforcing constraints only at the posterior mean $\bar{\btheta}$. While computationally convenient, this formulation does not explicitly control the probability of constraint violation under the full posterior distribution of $\btheta$. 

\subsubsection{Chance constraints risk functional}
\label{sec:4-so}

To explicitly incorporate parameter uncertainty, we consider chance constraints~\citep{charnes1959chance, miller1965chance}, which require critical
conditions to hold with a specified posterior probability. Rather than enforcing feasibility for every sampled parameter, we require the stability condition to
hold with posterior probability at least $\eta$. We define the joint feasible set for a fixed decision $\vect g$ as
\[
\mathcal{C}(\vect g)\;:=\;\{\btheta:\; c_j(\vect g,\btheta)\le 0,
\, \forall j=1,\ldots,m\}.
\]
The chance-constrained formulation, denoted $J^\ast_{\text{CC}}(y, \xi)$, is
\begin{equation}
J^\ast_{\text{CC}}(y, \xi) = \min_{\vect g}\; J(\vect g)
\ \ \text{s.t.} \ \
\mathbb{P}\!\left(\btheta \in \mathcal{C}(\vect g)\mid y,\xi\right) \ge \eta,
\label{eq:chance-constraint-theta}
\end{equation}
where $\eta \in (0,1)$ is the desired confidence level.

To turn \cref{eq:chance-constraint-theta} into a tractable differentiable convex program, we approximate it with a weighted scenario approximation built 
from the posterior surrogate. Given posterior samples $\{\btheta_i\}_{i=1}^N \sim q_\phi(\btheta\mid y,\xi)$ with self-normalized IS weights $\{\tilde{w}_i\}_{i=1}^N$ from \cref{sec:posterior}, we select a subset of scenarios $\mathcal{S}(y,\xi) \subseteq \{1,\ldots,N\}$ whose cumulative posterior weight covers at least the confidence level, 
$\sum_{i \in \mathcal{S}(y,\xi)} \tilde{w}_i \;\ge\; \eta,$
and enforce the constraints jointly on all retained scenarios:
\begin{equation}
J^\ast_{\text{CC}}(y,\xi)
\;\approx\;
\min_{\vect g}\; J(\vect g)
\ \ \text{s.t.}\ \
c_j(\vect g,\btheta_i)\le 0,
\ \ \forall\, i\in\mathcal{S}(y,\xi),\ j=1,\ldots,m.
\label{eq:cc-scenario}
\end{equation}

The scenario set $\mathcal{S}(y,\xi)$ is selected once per $(y,\xi)$ using the weighted coverage rule in \Cref{app:cc-saa} and held fixed during the inner optimization over $\vect g$. Because $J(\vect g)$ is convex in $\vect g$ and each $c_j(\cdot,\btheta_i)$ is convex in $\vect g$ by assumption, \cref{eq:cc-scenario} is a convex program whose solution map admits implicit differentiation through its KKT conditions, consistent with the differentiable decision layer used in \cref{sec:decision-layer}. 
\vspace{-1mm}
\subsubsection{Conditional Value-at-Risk (CVaR) functional}
\label{sec:cvar-method}
\vspace{-1mm}
As an alternative to the scenario-based chance constraint in \cref{sec:4-so}, we control the \emph{tail} of constraint violations using CVaR. For each constraint $c_j(\vect g,\btheta)$, let $\tau_j(\vect g)$ denote the posterior $\eta$-quantile of $c_j(\vect g,\btheta)$ under $\btheta\sim p(\btheta\mid y,\xi)$. $\mathrm{CVaR}_\eta(c_j)$ is the expectation of $c_j(\vect g,\btheta)$ restricted to the upper tail $\{\btheta : c_j(\vect g,\btheta) \ge \tau_j(\vect g)\}$, and we require $\mathrm{CVaR}_\eta(c_j)\le 0$. Because this definition is non-smooth in $\vect g$, we employ the Rockafellar-Uryasev variational representation~\citep{rockafellar2002conditional}, 
\[
\mathrm{CVaR}_\eta\!\big(c_j(\vect g,\btheta)\big)
\;=\;
\min_{\tau_j\in\mathbb{R}}\
\Big\{\,\tau_j + \tfrac{1}{1-\eta}\,
\mathbb{E}\!\left[(c_j(\vect g,\btheta)-\tau_j)_+\right]\Big\},
\]
which lifts $\tau_j$ to a decision variable and yields a jointly convex, differentiable feasible region in $(\vect g,\tau_j)$. The resulting formulation is
\begin{equation} 
\label{eq:cvar-problem-general}
J^\ast_{\text{CVaR}}(y,\xi) = \min_{\vect{g},\,\{\tau_j\}}\; J(\vect{g})
\ \ \text{s.t.}\ \
\tau_j + \tfrac{1}{1-\eta}\,\widehat{\mathbb{E}}\!
\left[(c_j(\vect g,\btheta)-\tau_j)_+\right] \le 0,
\ j=1,\ldots,m,
\end{equation}
where $\widehat{\mathbb{E}}[\cdot]$ denotes the self-normalized IS estimator in \cref{sec:posterior}. The sample-average epigraph reformulation used to implement \cref{eq:cvar-problem-general} as a tractable convex program is given in \Cref{app:cvar-saa}.

The three optimal value functions $J^\ast_{\text{mean}}$, $J^\ast_{\text{CC}}$, and $J^\ast_{\text{CVaR}}$ correspond to the mean-based, chance-constrained, and CVaR-based formulations, respectively. Each is precisely the quantity GoBOED seeks to reduce in \cref{eq:opt-design}, but doing so requires differentiating through the convex decision layer with respect to $\xi$, which we describe next.
\vspace{-1mm}
\subsection{Differentiating the robust optimal value}
\label{sec:decision-layer}
\vspace{-1mm}
To optimize the design $\xi$ for the objectives defined above, we embed the robust control problem as a \emph{differentiable decision layer} inside GoBOED and propagate gradients back to $\xi$ through the frozen posterior surrogate and the convex decision layer. Using the reparameterization trick, each posterior sample is generated via a differentiable path $\btheta_i = h(\epsilon_i; y,\xi,\phi)$ with base noise $\epsilon_i \sim p(\epsilon)$. The total derivative of the robust optimal cost with respect to $\xi$ follows by the multivariate chain rule across all posterior samples, 
\begin{equation}
\label{eq:dJdxi}
\frac{d J^\ast}{d\xi}
= \frac{\partial J^\ast}{\partial \xi}
+ \sum_{i=1}^N \big(\nabla_{\btheta_i} J^\ast\big)^{\!\top}
  \frac{\partial \btheta_i}{\partial \xi},
\end{equation}
where $\partial J^\ast/\partial \xi$ vanishes because $\xi$ enters $J^\ast$ only through the posterior samples $\btheta_i$. The term $\nabla_{\btheta_i} J^\ast$ is obtained via implicit differentiation of the KKT conditions using \texttt{cvxpylayers}~\citep{agrawal2019differentiable}. By the envelope theorem (\cref{app:envelope}), it equals ${\lambda_i^*}^\top \nabla_{\btheta_i} c(\vect g^*, \btheta_i)$, where $\lambda_i^*$ are the KKT multipliers at the optimum. The term $\partial \btheta_i / \partial \xi$ is the pathwise derivative through the amortized encoder (\cref{sec:two-stage}). The IS weights $\tilde w_i$ and scenario set $\mathcal{S}(y,\xi)$ are treated as constants in the backward pass to stabilize the gradient estimator.

\paragraph{Structural property of the decision-layer gradient.}
The gradient in \cref{eq:dJdxi} clarifies which posterior directions actually drive GoBOED design updates. When the constraints depend on $\btheta$ only through a low-dimensional task-relevant linear projection $P^\ast \in \mathbb{R}^{d\times d}$, that is, $c_j(\vect g, \btheta) = \tilde{c}_j(\vect g, P^\ast \btheta)$ for all $\vect g, \btheta, j$, the pathwise derivative through the decision layer is insensitive to the null subspace $I - P^\ast$. 

\begin{theorem}[Task-subspace alignment of the GoBOED gradient]
\label{thm:null-space-cancellation}
Let $P^\ast$ be the orthogonal projection onto the task-relevant subspace $\mathcal{S}^\ast = \mathrm{span}\{\nabla_{\btheta} c_j(\vect g,\btheta)\}$, and let $\btheta^{(i)} = h(y,\xi,\epsilon^{(i)})$ be reparameterized posterior samples
with IS weights treated as fixed in the computation graph. Under Assumptions ~\ref{asm:subspace} and ~\ref{asm:pathwise}, the pathwise gradient satisfies 
\begin{equation}
\label{eq:null-space}
\frac{\partial J^\ast(y,\xi)}{\partial \btheta^{(i)}}\,(I - P^\ast) \;=\; 0,
\qquad
\frac{d J^\ast(y,\xi)}{d \xi}
= \sum_{i=1}^{N}
\frac{\partial J^\ast}{\partial \btheta^{(i)}}\,
P^\ast\,
\frac{\partial \btheta^{(i)}}{\partial \xi}.
\end{equation}
Consequently, posterior uncertainty reduction in the null subspace $(I-P^\ast)\btheta$ contributes zero pathwise gradient signal, and any EIG increase achieved purely by reducing $\mathrm{Var}[(I-P^\ast)\btheta\mid y,\xi]$ leaves the local GoBOED objective unchanged. 
\end{theorem}

The theorem and its proof are given in \cref{app:structural}. \Cref{thm:null-space-cancellation} establishes this alignment for the
linear projection case. In the non-linear setting, the constraint Jacobian $\nabla_{\btheta} c_j(\vect g^\ast, \btheta^{(i)})$ varies with each posterior sample and a closed-form cancellation no longer holds in general. We treat a formal non-linear extension as future work. The central mechanism is validated empirically in \cref{sec:result}.
\vspace{-3mm}
\subsection{Two-stage optimization}
\label{sec:two-stage}
\vspace{-3mm}
Computing the gradient in \cref{sec:decision-layer} requires both the posterior samples $\btheta_i$ and their sensitivities $\partial\btheta_i/\partial\xi$, both of which are provided by an amortized variational posterior surrogate. While one could optimize $(\xi,\phi)$ jointly~\citep{NEURIPS2020_d3d94468,pmlr-v15-lacoste_julien11a}, the single-step setting allows us to decouple posterior learning from design optimization, keeping the downstream robust objective tractable.
\vspace{-3mm}
\paragraph{Stage 1: offline amortized VI.}

Motivated by \citet{huang2024amortized}, we train a single amortized variational network $q_\phi(\btheta\mid y,\xi)$ that maps a design and observation pair $(\xi,y)$ to a variational posterior, and then reuse the trained network during design optimization. The network is trained once by maximizing the ELBO in \cref{eq:SVI} averaged over simulated $(\xi,y)$ pairs to obtain $\phi^*$. The amortizer outputs variational parameters for a reparameterizable family, from which posterior samples are drawn via the reparameterization introduced in \cref{sec:decision-layer}. This yields stable Monte Carlo gradient estimates during training~\citep{burda2015importance,foster2020unified,foster2019variational} and enables the pathwise derivatives $\partial\btheta_i/\partial\xi$ used during design optimization.
\vspace{-3mm}
\paragraph{Stage 2: online design optimization.}

We freeze $\phi^*$ and optimize $\xi$ for downstream robust decision-making:
\begin{equation}
\label{eq:stage2-obj}
\xi^* = \arg\min_{\xi\in\Xi}\; L(\xi),
\qquad
L(\xi) := \mathbb{E}_{y \sim p(y\mid\xi)}\!\big[J^\ast(y,\xi;\phi^*)\big],
\end{equation}
where $J^\ast(y,\xi;\phi^*)$ denotes the optimal value of the robust control problem computed using the frozen posterior surrogate $q_{\phi^*}(\btheta\mid y,\xi)$. For notational simplicity, we drop the explicit dependence on $\phi^*$ and the specific risk functional, writing $J^\ast(y,\xi)$ throughout the remainder of the paper.

Combining the log-derivative trick on the outer expectation with the pathwise derivative from \cref{eq:dJdxi} yields the total gradient 
\begin{equation}
\label{eq:total-gradient}
\nabla_\xi L(\xi) = \mathbb{E}_{y \sim p(y\mid\xi)}\!\bigg[\,
\frac{d J^\ast(y,\xi)}{d\xi}
+ J^\ast(y,\xi)\,\nabla_\xi \log p(y\mid\xi) \,\bigg].
\end{equation}
The first term $dJ^\ast/d\xi$ follows from \cref{eq:dJdxi}. The second is the score-function estimator for any differentiable observation model $p(y\mid\btheta,\xi)$, with the specific form used in our experiments given in \cref{sec:marlikgrad}. Together, these estimators provide stable gradients for decision-focused design optimization. A high-level summary of the full pipeline is given in \cref{alg:goboed}.
\vspace{-3mm}
\paragraph{Computational benefit.}
 Amortization lets us learn a single network that maps $(\xi,y)$ to approximate posterior parameters once, and then reuse this surrogate during design optimization to obtain both posterior samples and their sensitivities without re-solving a full posterior inference problem for every candidate $\xi$ and every gradient step. This is critical because each candidate additionally requires solving a convex robust decision problem and backpropagating through its KKT conditions, so avoiding repeated VI approximations substantially reduces the overall cost of design optimization.
\vspace{-3mm}
\section{Results}\vspace{-3mm}
\label{sec:result}
We present numerical experiments on three use cases to evaluate goal-driven experimental design for robust decision-making: an intuitive 2D source-localization problem and two dynamical systems under model uncertainty, namely epidemic management based on the SIQR model and pharmacokinetic (PK) control. Detailed model specifications, and parameter values are provided in \Cref{app:models}.
\vspace{-3mm}
\paragraph{Scope of baselines.}

We compare against EIG-based BOED. Direct comparison with \citet{zhong2024goal}, \citet{huang2024amortized}, and BAX~\citep{neiswanger2021bayesian} is not apples-to-apples: the first targets a scalar predictive QoI via per-design MCMC, the second outputs a learned policy that cannot accommodate plug-in chance/CVaR constraints, and the third optimizes mutual information with a computable property rather than a robust control cost. Adapting any of them to constrained robust control would itself be a methodological contribution.

\vspace{-3mm}
\paragraph{Source localization with a single sensor.}
As an intuitive example, we consider a two-dimensional source-location problem in which a single sensor must be placed to localize an emitting source and support a downstream intervention decision. The unknown parameter $\btheta$ encodes the coordinates of a point source, and the design variable is the sensor location $\xi \in \mathbb{R}^2$. Following~\cite{foster2021deep}, we observe a noisy scalar intensity measurement of approximately $1/\|r\|^2$, where $r = \btheta - \xi$, with an angular observation $r/\|r\|$.

\begin{figure}[t]
\centering
\begin{subfigure}[t]{0.24\columnwidth}
  \centering
  \includegraphics[width=\linewidth]{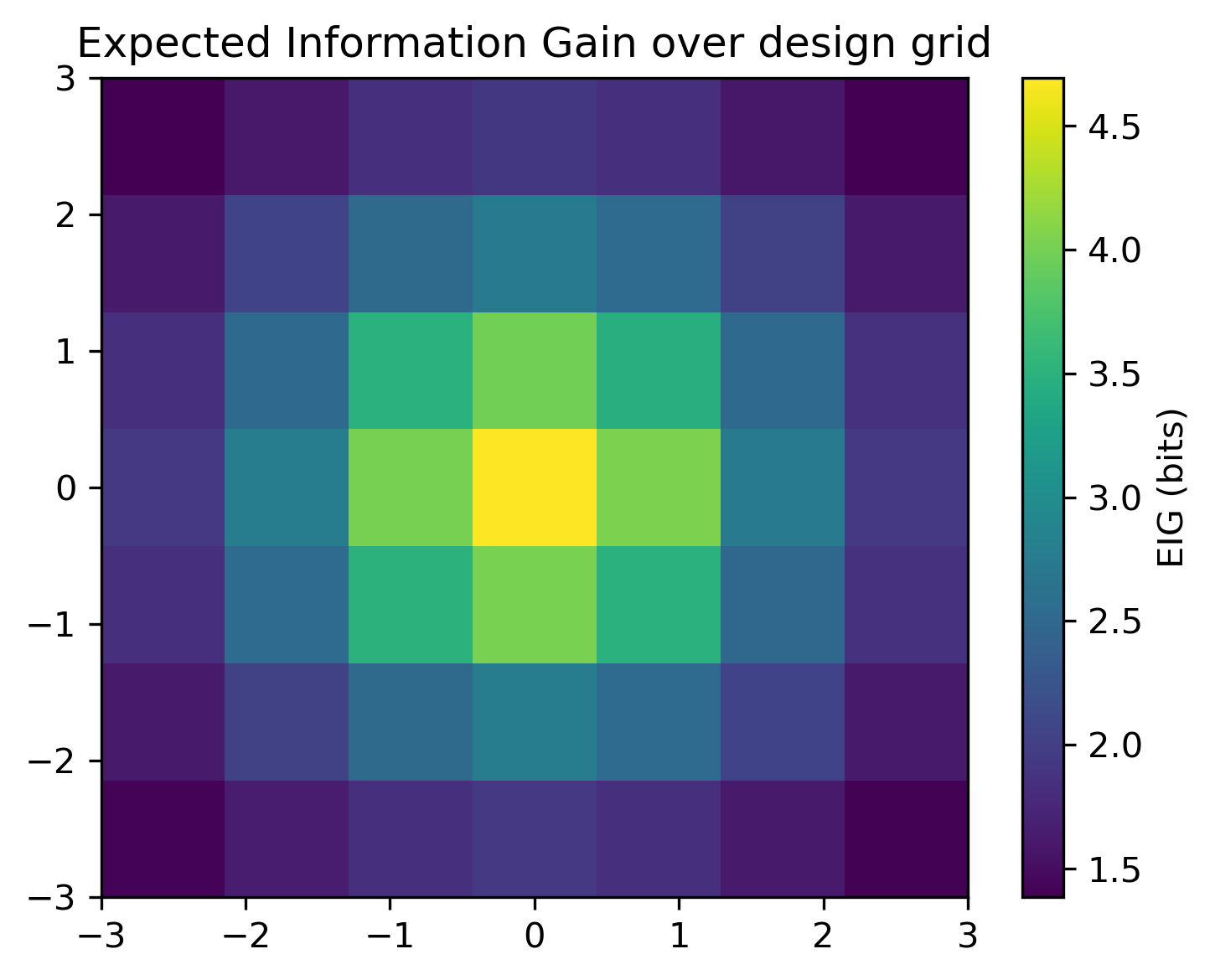}\vspace{-2mm}
  \caption{EIG over $\xi$.}
  \label{fig:loc-eig}
\end{subfigure}
\hfill
\begin{subfigure}[t]{0.24\columnwidth}
  \centering
  \includegraphics[width=\linewidth]{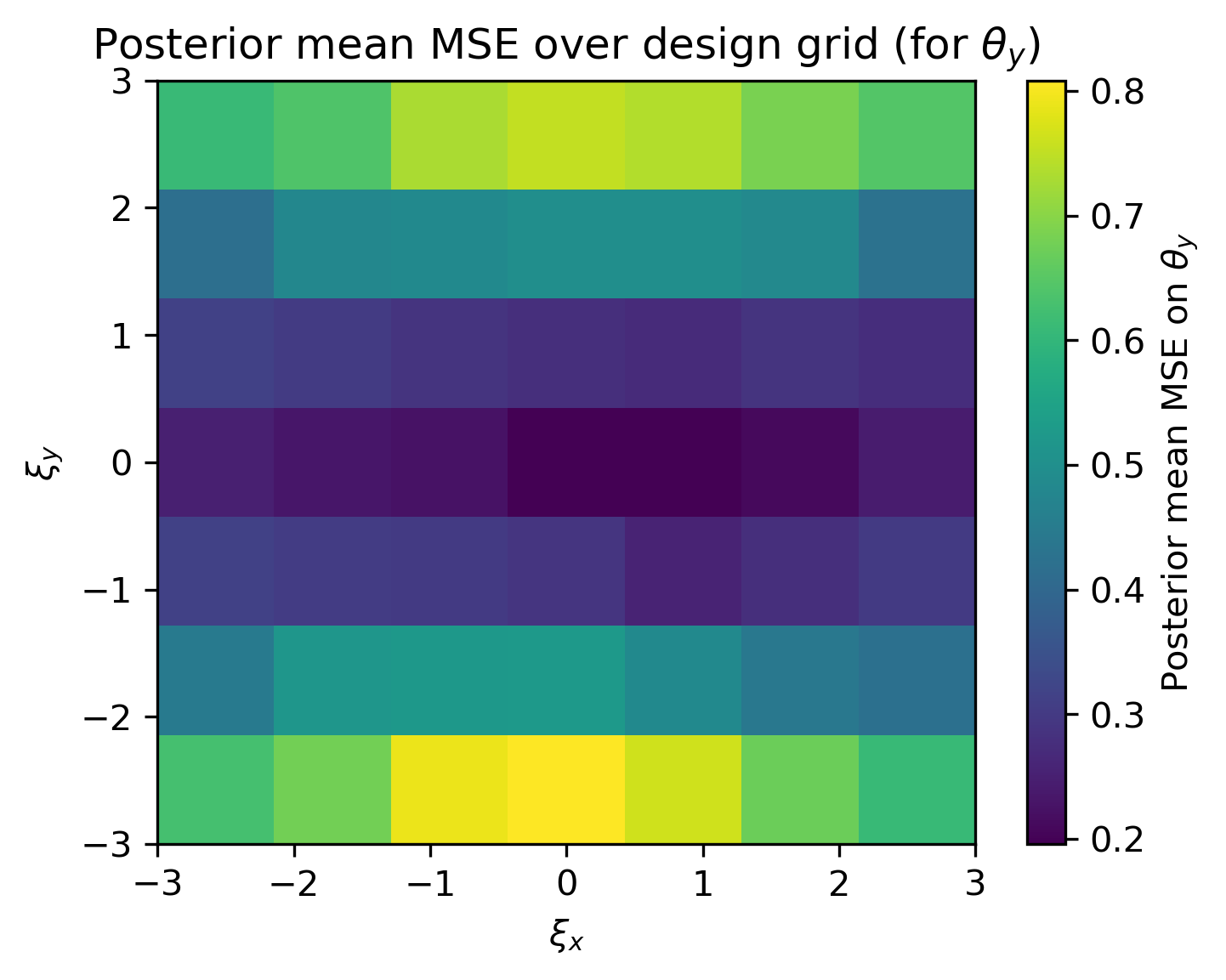}\vspace{-2mm}
  \caption{MSE for $\btheta_y$.}
  \label{fig:loc-mse}
\end{subfigure}
\hfill
\begin{subfigure}[t]{0.23\columnwidth}
  \centering
  \includegraphics[width=\linewidth]{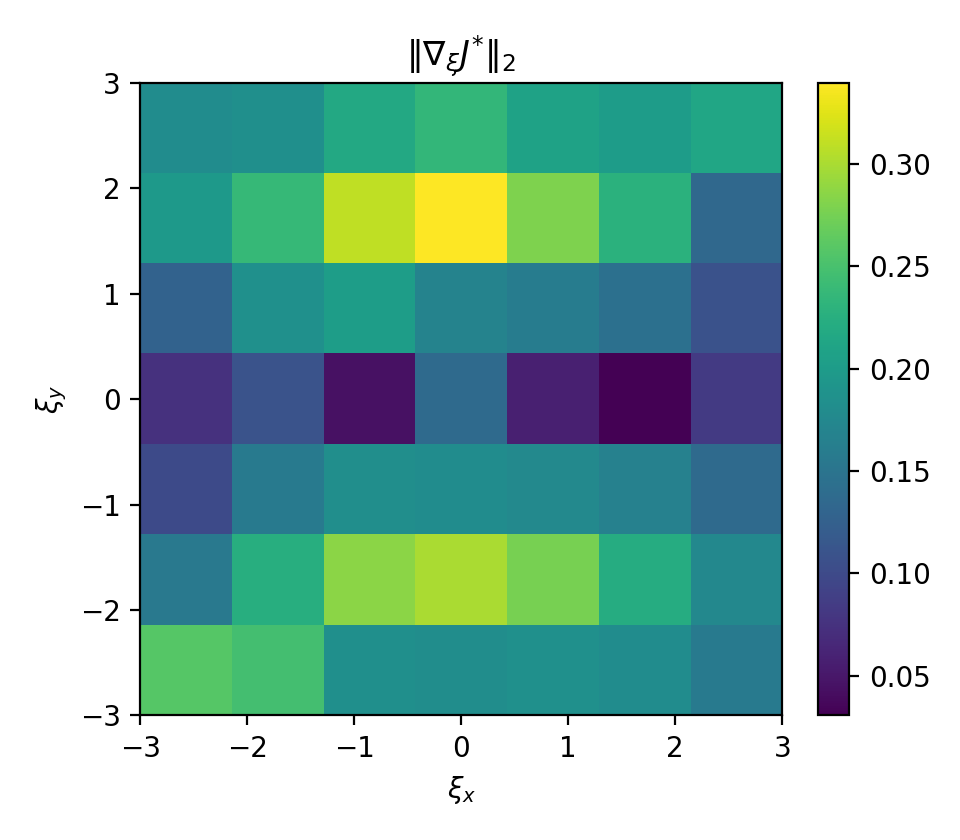}\vspace{-2mm}
  \caption{$\|\partial J^\ast/\partial \xi\|_2$.}
  \label{fig:dj-dxi}
\end{subfigure}
\hfill
\begin{subfigure}[t]{0.23\columnwidth}
  \centering
  \includegraphics[width=\linewidth]{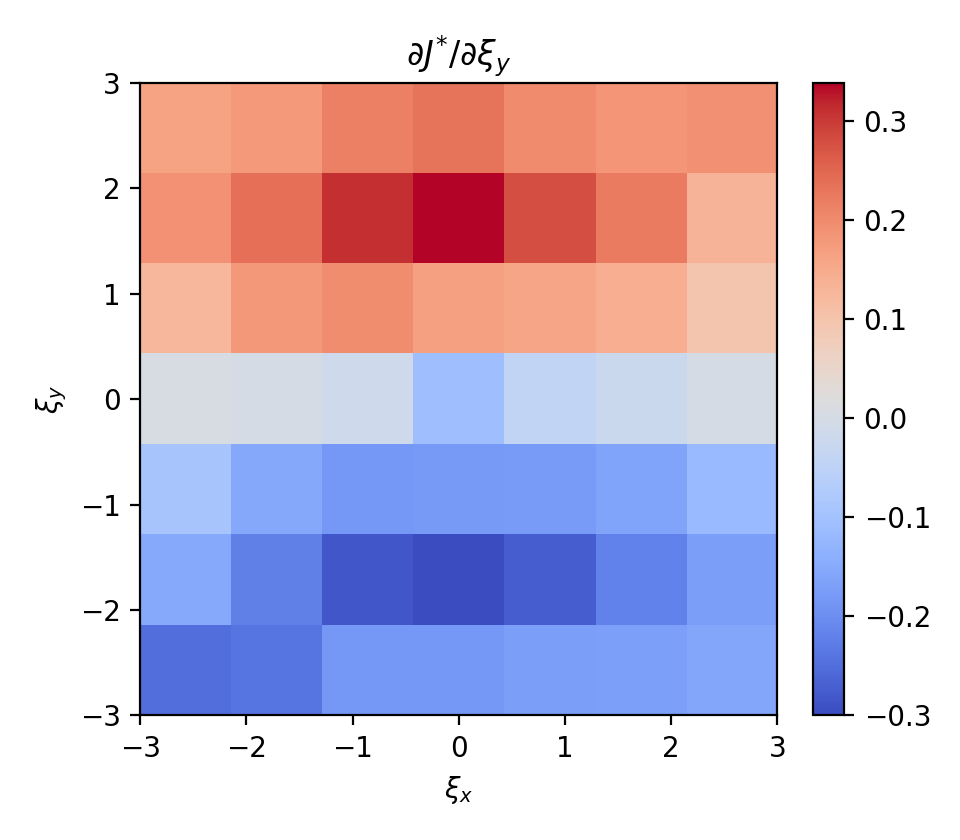}\vspace{-2mm}
  \caption{$\partial J^\ast/\partial \xi_y$.}
  \label{fig:dj-dxiy}
\end{subfigure}
\vspace{-1mm}
\caption{
  Source-location toy problem on a $7\times 7$ grid of sensor locations $\xi$.
  (a) EIG is maximized near the center of the domain.
  (b) The decision-focused metric (posterior mean MSE of the source $y$-coordinate) is minimized along a broad horizontal band, yielding a wider near-optimal region for GoBOED than for standard EIG-based design.
  (c) Gradient norm $\|\partial J/\partial \xi\|_2$, confirming that the gradient signal is concentrated near the center of the domain.
  (d) $y$-component of the gradient $\partial J/\partial \xi_y$, which is the dominant term consistent with the decision-focused objective.
}
\label{fig:loc-grid}
\vskip -0.2in
\end{figure}

On a $7\times 7$ grid of candidate sensor locations, we compare standard BOED, which maximizes EIG about the source position, with GoBOED, which optimizes a decision-focused objective measuring the accuracy of the recovered source's vertical coordinate. For EIG, larger values indicate more informative designs, whereas for the $\theta_y$ MSE objective, smaller values indicate better performance ($10{,}000$ outer samples and $10{,}000$ inner samples per design). 

\Cref{fig:loc-eig} shows the EIG surface over sensor locations, while \Cref{fig:loc-mse} displays the corresponding posterior mean squared error for $\btheta_y$. The EIG peaks sharply at the center of the domain, whereas the decision-focused objective exhibits a broad horizontal valley of near-optimal designs. We further validate \cref{thm:null-space-cancellation} by examining the gradient. \Cref{fig:dj-dxi} shows that the gradient norm $\|\partial J/\partial \xi\|_2$ is dominated by the $y$-component (\Cref{fig:dj-dxiy}), consistent with \cref{thm:null-space-cancellation} that only the decision-relevant coordinate drives the design gradient. 

This example mirrors the qualitative behavior observed in our larger case studies, but with a simple loss. We now turn to SIQR and PK models, where the downstream decision is a higher-dimensional optimal control policy obtained by solving a convex program at each candidate design, and where the control actions influence the nonlinear system dynamics in a complex way.

\vspace{-3mm}
\paragraph{GoBOED with SIQR and PK models.} To compare GoBOED with standard BOED baselines, we study the problem of choosing a single observation time. Let $T$ denote the time horizon and let $\boldsymbol{\xi}\in\{0,1\}^T$ be a one-hot design vector with $\sum_{t=1}^T \xi_t=1$.
The unique index $t^*$ with $\xi_{t^*}=1$ is the chosen measurement time, at which we collect one datum and update the posterior over model parameters. This formulation applies to both the SIQR and PK settings, as both are time-indexed. At $t^*$, we observe the counts of asymptomatic and symptomatic infections for the SIQR model and the drug concentration in blood for the PK model, with measurement noise modeled as Poisson or Gaussian, respectively.

BOED via EIG targets maximal reduction in parameter uncertainty, whereas the robust decision-making objective is sensitive to particular parameter combinations and system dynamics. Accordingly, for each candidate design $\boldsymbol{\xi}$ we compute and visualize both the EIG and the robust optimal control cost over the entire design space. \Cref{fig:combined-fig} visualizes the resulting design objective
surfaces for the SIQR (top) and PK (bottom) models, where higher EIG and lower control cost indicate better designs. For SIQR, the control cost represents the economic cost under the optimal quarantine rates, and for PK, it represents the minimum dose required to reach the target concentration. For infeasible prior samples, where no feasible quarantine rate (SIQR) or optimal dose (PK) exists under the chance constraint, we pad the
cost with the mean over feasible trajectories. A principled penalization is left to future work, as our focus here is on the overall sample trend.

We estimate EIG (cf. \cref{eq:eig}) using nested Monte Carlo with $5{,}000$ outer samples (over $y$) and $3{,}000$ inner samples for the marginal likelihood. The BOED-selected optimal observation times are day~\textbf{5} for the SIQR model and hour~\textbf{17} for the PK model.

\vspace{-3mm}
\paragraph{Robust optimal control in the presence of model uncertainty.} We solve the chance-constrained problem in \cref{eq:chance-constraint-theta}
under the posterior induced by a given observation time, requiring a $90\%$ probability of constraint satisfaction for SIQR and $80\%$ for PK. For SIQR, the chance-constrained 
objective is relatively flat for observation times between days 7 and 10
(\Cref{fig:combined-fig}). For PK, observing later (roughly 15--23 hours)
reduces the dose required to meet therapeutic targets. For both models,
the robust optimal control objective selects designs that differ from the
EIG-optimal one and yields a broader near-optimal window (a range of observation times achieving objective values close to the optimum), providing
flexibility in the choice of observation time. We further demonstrate via
gradient-based search that GoBOED recovers designs within this window
(\Cref{fig:combined-fig}). We draw $500$ datasets $y$ and, for each $y$,
$40$ posterior samples of $\bm{\theta}$ (see \Cref{app:mc-budgets}).

\vspace{-3mm}
\paragraph{CVaR-based constraints.} 
We also consider the CVaR formulation (\Cref{eq:cvar-problem-general}), setting $\eta = 0.9$ for SIQR and $\eta = 0.7$ for PK. The qualitative trends mirror those under chance constraints across both models (\Cref{fig:combined-fig}), confirming that the broad near-optimal windows are robust to the choice of risk functional. We use the same sampling budgets as in the chance-constrained
experiments (see \Cref{app:mc-budgets}).
 
\begin{figure*}[t]
    \centering

    \begin{subfigure}[b]{0.32\textwidth}
        \includegraphics[width=\textwidth]{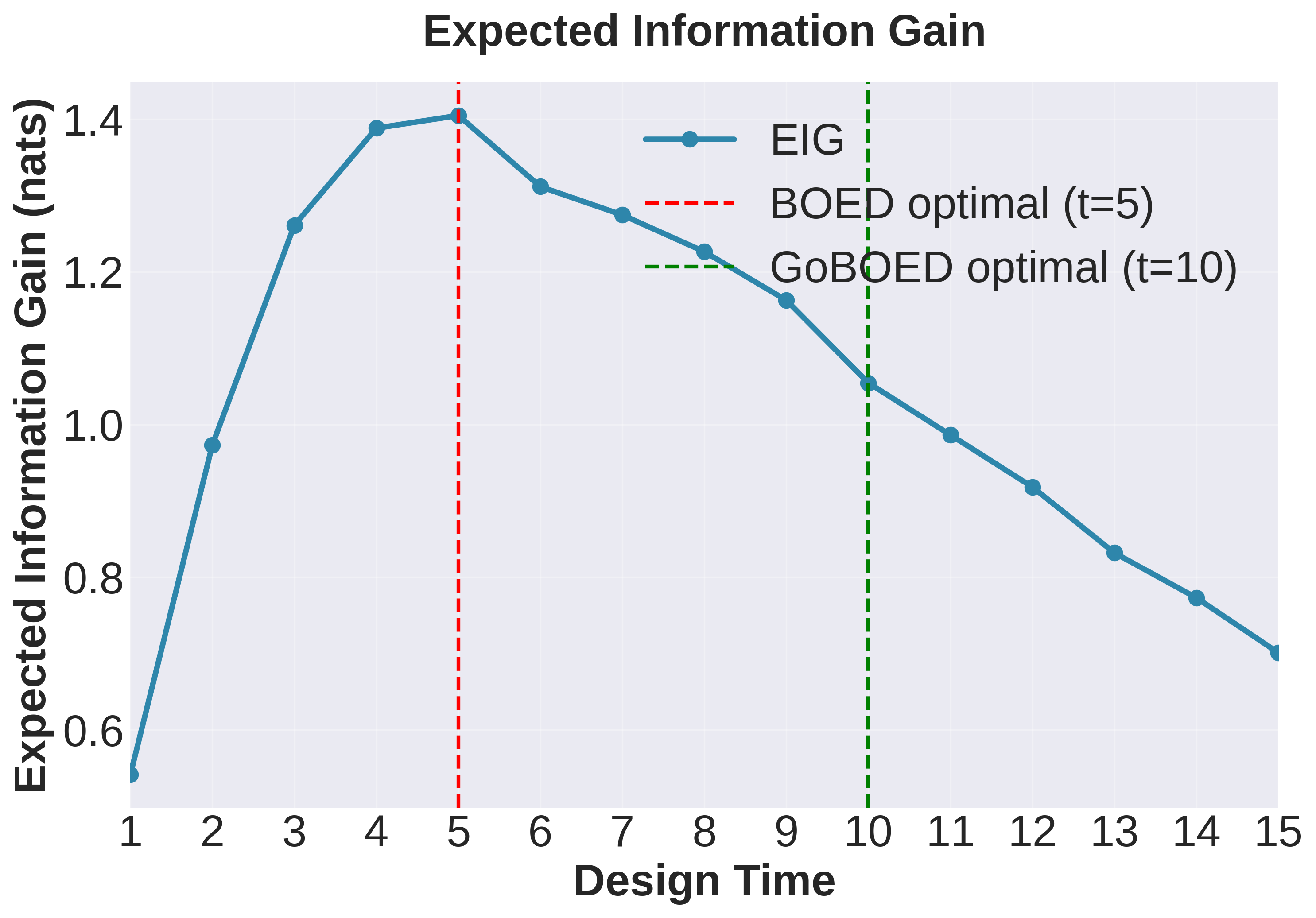}\vspace{-1mm}
        \caption{EIG (SIQR)}
    \end{subfigure}\hfill
    \begin{subfigure}[b]{0.32\textwidth}
        \includegraphics[width=\textwidth]{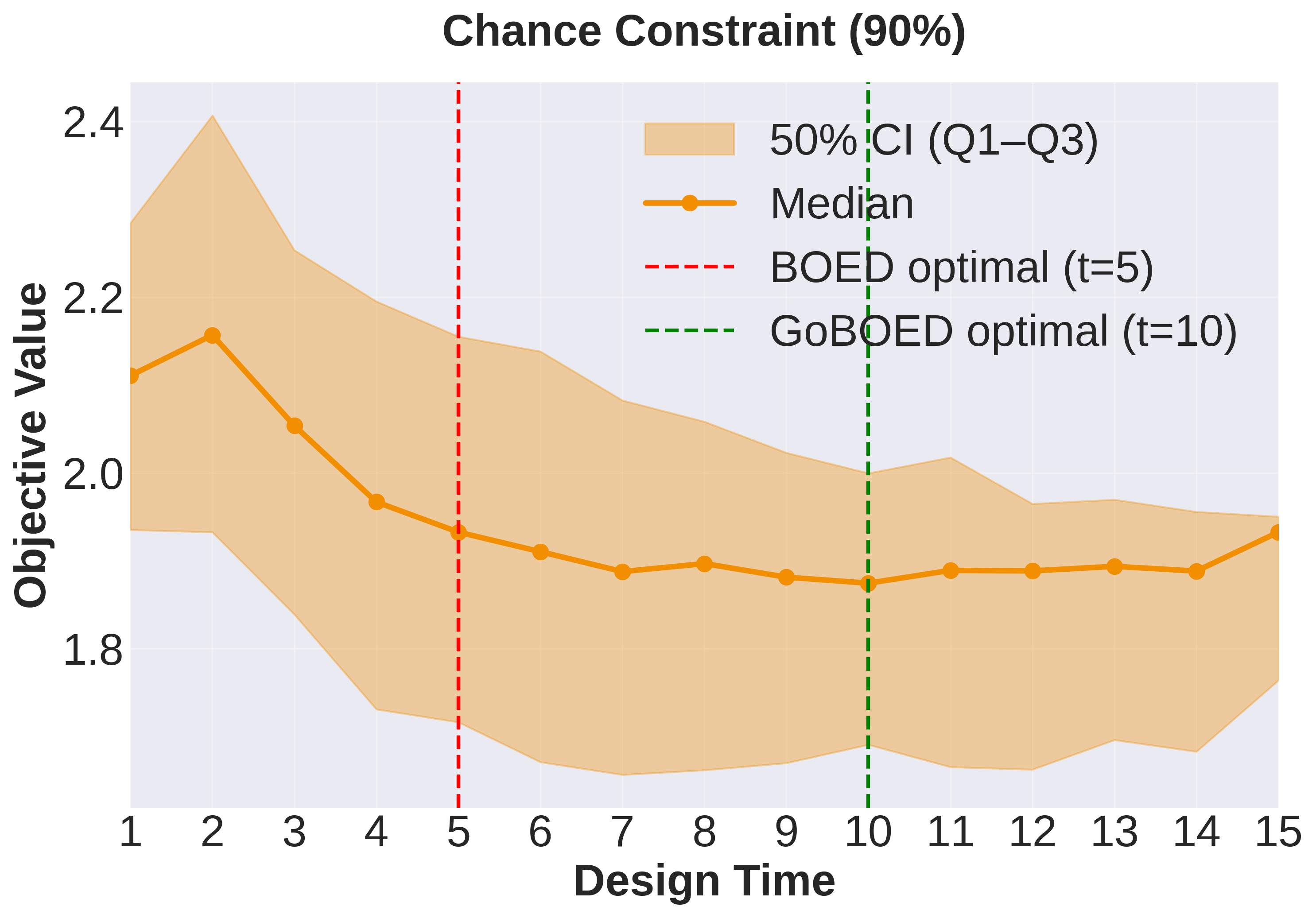}\vspace{-1mm}
        \caption{Chance constraints (SIQR)}
    \end{subfigure}\hfill
    \begin{subfigure}[b]{0.32\textwidth}
        \includegraphics[width=\textwidth]{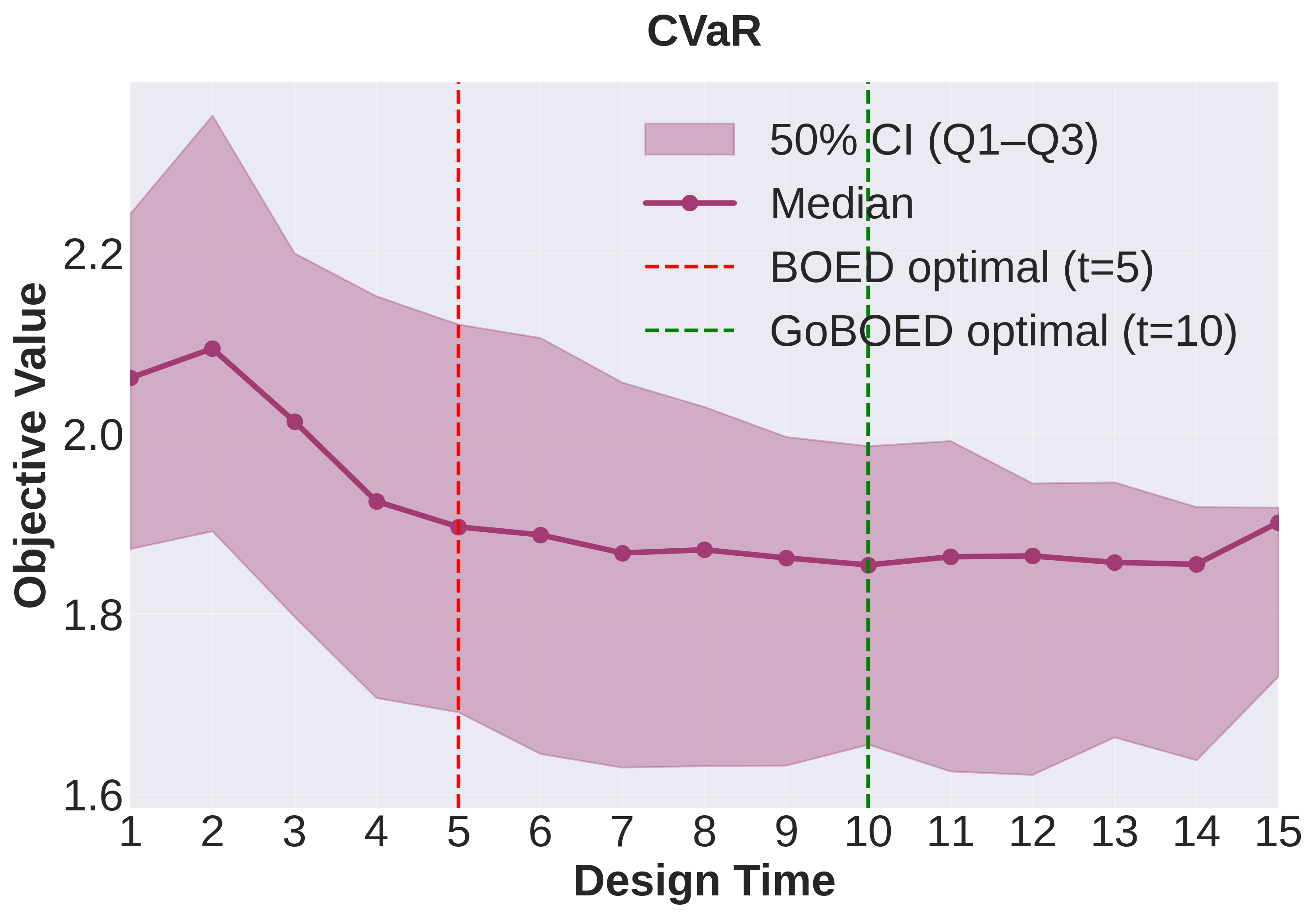}\vspace{-1mm}
        \caption{CVaR (SIQR)}
    \end{subfigure}

    \vspace{2mm}

    \begin{subfigure}[b]{0.32\textwidth}
        \includegraphics[width=\textwidth]{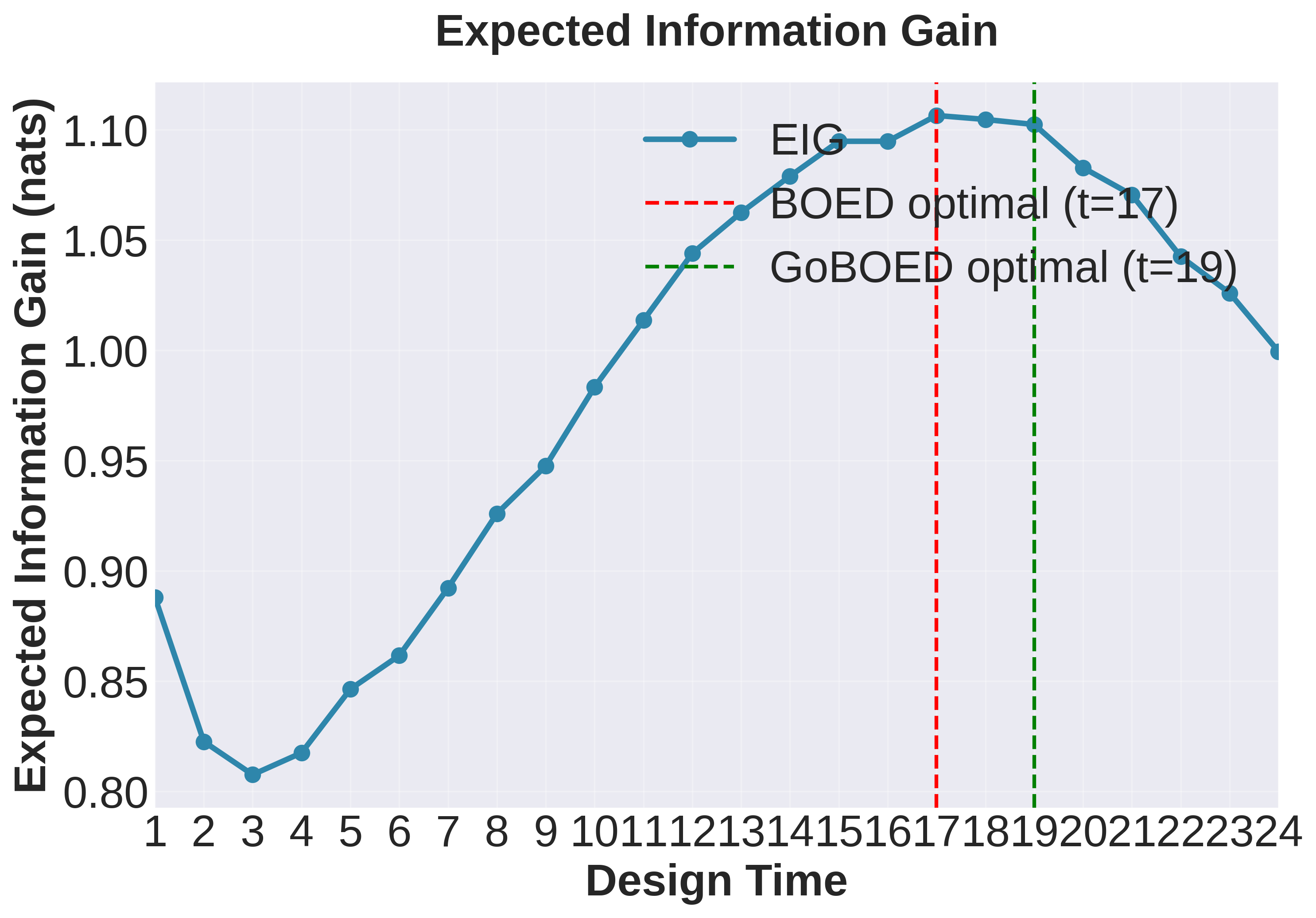}\vspace{-1mm}
        \caption{EIG (PK)}
    \end{subfigure}\hfill
    \begin{subfigure}[b]{0.32\textwidth}
        \includegraphics[width=\textwidth]{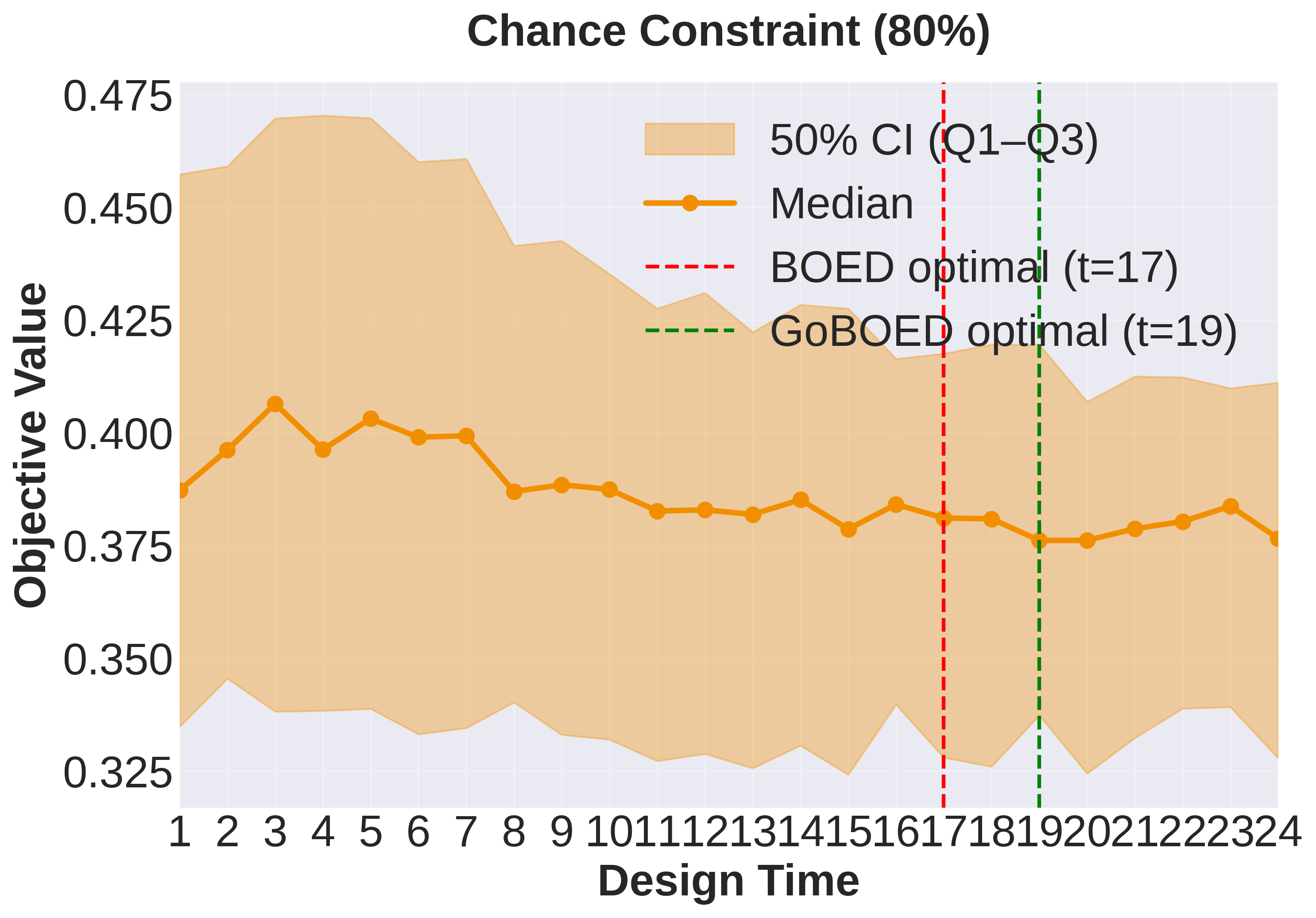}\vspace{-1mm}
        \caption{Chance constraints (PK)}
    \end{subfigure}\hfill
    \begin{subfigure}[b]{0.32\textwidth}
        \includegraphics[width=\textwidth]{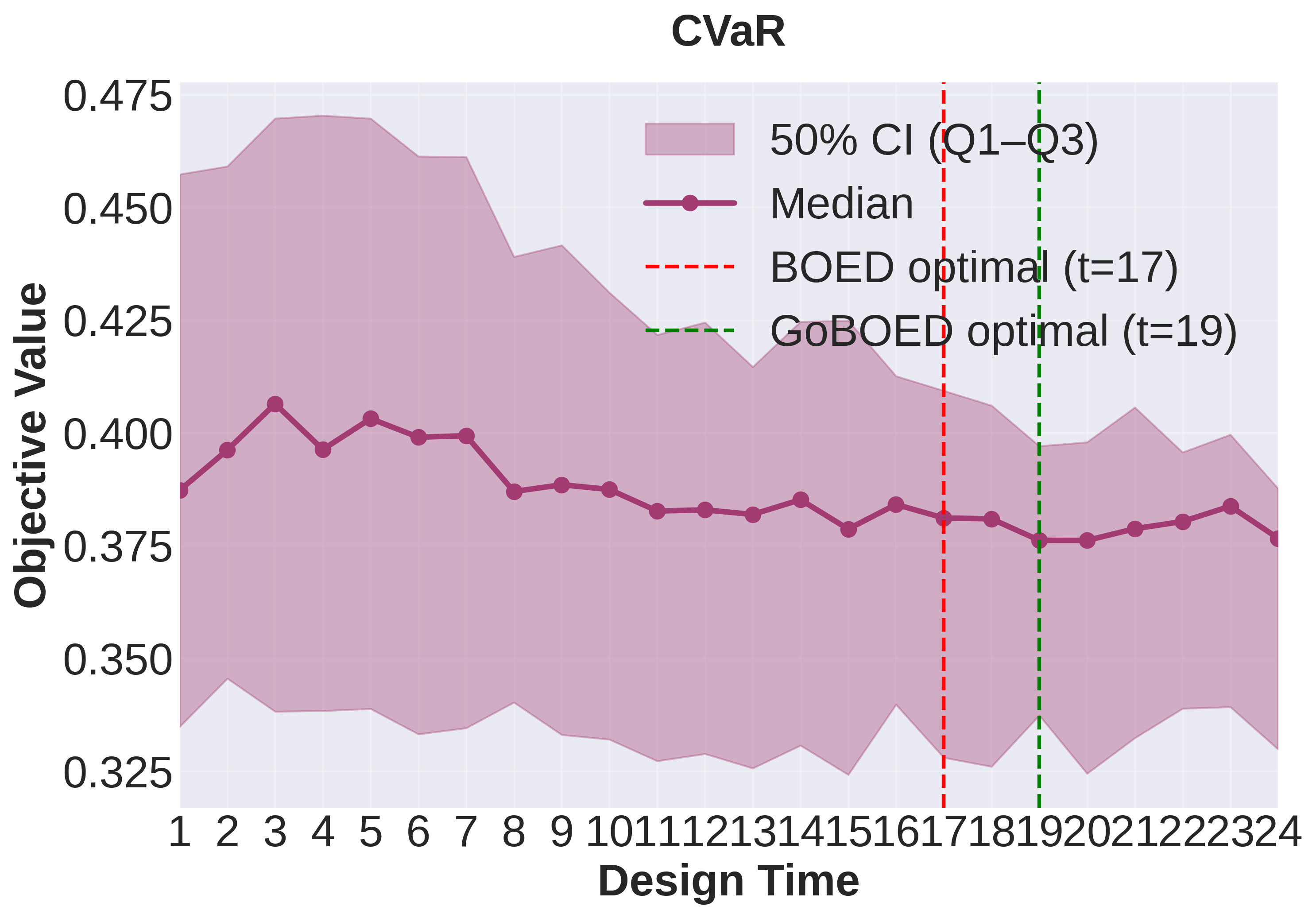}\vspace{-1mm}
        \caption{CVaR (PK)}
    \end{subfigure}

    \vspace{-1mm}
    \caption{Comparison of experimental design metrics and control strategies across two models. Top row: SIQR epidemiological model: (a) EIG over observation time $\xi$, (b) expected optimal cost under chance constraints, (c) expected optimal cost under CVaR. 
    Bottom row: PK model: (d) EIG, (e) chance constraints, (f) CVaR. The horizontal axis is observation time $\xi$. The BOED-optimal design typically pinpoints a single time, whereas the goal-driven robust objective admits a broader near-optimal window, offering greater scheduling flexibility under real-world constraints.} \vspace{-5mm}
    \label{fig:combined-fig}
\end{figure*}

These results highlight a key insight: maximizing EIG alone does not necessarily optimize constraint-aware economic performance. Instead, our approach identifies a broad near-optimal window for
SIQR, approximately days~7--11, within which observation timing can be chosen with minimal loss in both information gain and economic efficiency. For PK dose optimization, we similarly find a wide near-optimal window (about 15--23~hours). Later observations (toward 22--23~hours) generally yield lower robust cost than the BOED-optimal 17~hours while achieving comparable EIG. This scheduling flexibility is practically valuable when exact observation timing is restricted by logistical or operational constraints, and reflects the central insight of GoBOED: aligning experimental design with the downstream decision objective reveals a broader set of near-optimal designs than standard EIG maximization. 
\vspace{-3mm}
\paragraph{Posterior visualization.}
We provide marginal posterior plots under the EIG-optimal and GoBOED designs for both models in \Cref{sec:additionalvis}. GoBOED produces sharper and shifted posteriors along the parameter directions most influential for the downstream control objective, while leaving uncertainty on other parameters comparable to the EIG-optimal design.

\vspace{-3mm}
\paragraph{Gradient-based search over discrete times.}

With a slight abuse of notation, we treat the design $\xi \in [1, T]$ as a scalar and optimize it directly with automatic differentiation (AD). Starting from the midpoint $\xi_0 = \lfloor (T+1)/2 \rfloor$, we compute the AD gradient of the robust objective $J(\xi)$ and take projected gradient steps, 
\[
s_k \;=\; \frac{\partial J(t)}{\partial \xi}\Big|_{\xi=\xi_k}, \
\xi_{k+1}\;=\;\Pi_{[1,T]}\!\bigl(\xi_k-\omega_k\, s_k\bigr), \
\hat \xi_{k+1}\;=\;\mathrm{round}(\xi_{k+1}),    
\]
where $\Pi_{[1,T]}$ clips to the feasible interval and $\omega_k$ is the step size.  After each update we evaluate the design at the integer index $\hat \xi_{k+1}$ by recomputing the posterior induced by a single observation at $\hat \xi_{k+1}$ and re-evaluating $J(\hat \xi_{k+1})$. We repeat these steps until convergence.

For SIQR, the procedure converges to a design in the near-optimal window, with $\xi = 10$ and $\xi \in \{7, 11\}$ yielding similarly small gradient norms. For PK, it selects $\xi = 22$ under both the CVaR and chance-constrained criteria, consistent with \Cref{fig:combined-fig}. The gradient norm decreases monotonically across iterations.
\vspace{-3mm}

\paragraph{Computational cost}

\begin{wraptable}{r}{0.48\textwidth}

\vspace{-4mm}

\caption{Online wall-clock cost per robust-objective evaluation. Offline amortizer training excluded.}

\label{tab:runtime}

\vspace{-1mm}
\centering
\scalebox{0.96}{
\begin{tabular}{llrrr}
\toprule
Task & Risk & Fwd (s) & Diff (s) & Total (s) \\
\midrule
SIQR & CC   & 350.86 & 35.67 & 386.53 \\
SIQR & CVaR & 352.85 & 42.11 & 394.96 \\
PK   & CC   &   4.10 &  1.29 &   5.39 \\
PK   & CVaR &  15.46 &  0.62 &  16.08 \\
\bottomrule
\end{tabular}}
\vspace{-3mm}

\end{wraptable}

Table~\ref{tab:runtime} reports the online wall-clock cost of a single robust-objective evaluation with the amortized posterior fixed. For SIQR, one evaluation costs approximately 387–395 seconds, dominated by the setup/forward pass ($\approx$351–353 s). Exhaustive enumeration over the 15 evaluated SIQR candidate observation days would therefore require roughly $15\times387\approx5{,}805$ seconds. For PK, per-evaluation cost is lower (5.4 s for CC, 16.1 s for CVaR), and across both tasks our gradient-based optimizer converges in 3–5 projected gradient steps, which requires only a small multiple of a single evaluation rather than a full grid sweep. The posterior surrogate is trained once offline and reused across all gradient steps, so the online cost reduces to a small number of convex solves and KKT backward passes.
\vspace{-3mm}
\section{Conclusion}
\label{sec:conclusion}
\vspace{-3mm}

We introduced \textbf{GoBOED}, a goal-driven BOED framework that couples offline amortized posterior learning with decision-focused optimization of experimental designs under posterior uncertainty. By pairing a reusable posterior surrogate with a differentiable convex decision layer equipped with plug-in risk functionals (chance, CVaR), GoBOED selects experiments based on downstream robust decision quality rather than information gain alone. We further showed that the resulting design gradient is provably  insensitive to task-irrelevant parameter directions (\cref{thm:null-space-cancellation}). Across source localization, epidemic management (SIQR), and pharmacokinetic control, GoBOED yields broader near-optimal design windows and better control outcomes than standard EIG-based design.

We focused on a single-step instantiation in which the posterior surrogate is trained offline and the design is optimized separately, a decoupling that keeps gradient-based optimization tractable and the pipeline modular, but is in principle compatible with end-to-end joint training. Our analysis assumes a convex downstream decision layer, which enables sensitivity analysis via Lagrangian differentiation. Non-convex objectives generally admit only local guarantees. Performance also depends on the quality of the variational posterior, since inaccurate inference can degrade downstream decisions. Extending GoBOED to fully sequential adaptive designs, richer posterior families such as diffusion- and flow-based models, and end-to-end joint training are promising future directions~\citep{huang2024amortized, foster2021deep, ivanova2021implicit, ivanova2024stepdad}.

\bibliographystyle{plainnat}
\bibliography{example_paper}


\appendix

\section{Notations}
\label{app:notation}

Table~\ref{tab:notation} provides the detailed list of mathematical notations and symbols adopted in GoBOED. 

\begin{table}[h!]
\centering
\small
\begin{tabular}{ll}
\toprule
\textbf{Symbol} & \textbf{Description} \\
\midrule
\multicolumn{2}{l}{\textbf{Bayesian model and experimental design}} \\
\midrule
$\btheta \in \Theta$      & Unknown model parameters (e.g., SIQR or PK parameters) \\
$\xi \in \Xi$             & Experimental design (e.g., observation time) \\
$y \in \mathcal{Y}$       & Observation generated under design $\xi$ \\
$p(\btheta)$              & Prior over model parameters \\
$p(y \mid \btheta,\xi)$   & Likelihood (forward model plus observation noise) \\
$p(y \mid \xi)$           & Marginal likelihood under design $\xi$ \\
$f(\btheta,\xi)$          & Forward model mapping $(\btheta,\xi)$ to noiseless output \\
$\epsilon$                & Observation noise term \\
$\mathrm{EIG}(\xi)$       & Expected information gain objective for design $\xi$ \\
$D_{\mathrm{KL}}(\cdot\Vert\cdot)$ & Kullback--Leibler divergence \\
\midrule
\multicolumn{2}{l}{\textbf{Variational inference and amortizer}} \\
\midrule
$q_\phi(\btheta \mid \xi,y)$
  & Variational approximation to $p(\btheta \mid \xi,y)$ with parameters $\phi$ \\
$\phi$                    & Parameters of the amortized VI network \\
$\psi_\phi(\xi,y)$        & Variational parameters output by the amortizer \\
$\mathcal{L}_{\text{VI}}(\phi; y,\xi)$
  & Evidence lower bound (ELBO) for $(y,\xi)$ \\
$h_\phi(\varepsilon;\psi_\phi(\xi,y),\xi,y)$
  & Reparameterization map turning base noise $\varepsilon$ into $\btheta$ samples \\
$w_i, \tilde{w}_i$        & Importance weights and their normalized versions for sample $i$ \\
\midrule
\multicolumn{2}{l}{\textbf{Robust control and risk measures}} \\
\midrule
$\vect{g}$ & Decision / control variables (e.g., quarantine rates or dose) \\
$J(\vect{g} )$      & Application cost for decision $\vect{g}$ \\
$\rho(\cdot)$              & Risk functional (expectation, chance constraint, or CVaR) \\
$\eta$                     & Confidence / risk level for chance constraints or CVaR \\
$\tau$                     & Auxiliary threshold variable in CVaR formulation \\
$s_i$                      & Slack variable for CVaR constraint for sample $i$ \\
$\lambda_i$                & Lagrange multiplier for constraint of sample $i$ (decision layer) \\
\midrule
\multicolumn{2}{l}{\textbf{Optimization and design objective}} \\
\midrule
 $L(\xi)$ & Design loss based on the expected robust optimal value with frozen posterior surrogate \\
$c_j(\vect{g};\btheta) \le 0, \quad j=1,\ldots,m$  & Constraints on control variable for posterior sample $\btheta$ \\
$\mathcal{C}(\vect g)$  & Joint feasible parameter set for decision $\vect g$ \\
$\tau_j(\vect g)$ & $\eta$-quantile (VaR) threshold of $c_j(\vect g;\btheta)$\\
$\mathcal{D}_j(\vect g)$ & Upper-tail (violation) set of parameters for constraint $j$ under $\vect g$ \\
$\omega_\xi, \omega_\phi$  & Step sizes for updating $\xi$ and $\phi$ in gradient-based optimization \\
$t$                        & Discrete observation time index (when $\xi$ encodes timing) \\
\bottomrule
\end{tabular}
\caption{Summary of notations used for Bayesian optimal experimental design (BOED), variational
inference, and robust decision-making in GoBOED.}
\label{tab:notation}
\end{table}

\subsection{Self-normalized importance-sampling estimator}
\label{app:is-estimator}
Given samples $\{\theta^{(i)}\}_{i=1}^{N} \sim q_\phi(\theta \mid y, \xi)$
with unnormalized weights
$w_i = p(y|\theta^{(i)},\xi)\,p(\theta^{(i)})\;/\;q_\phi(\theta^{(i)}|y,\xi)$,
define $\tilde{w}_i = w_i / \sum_j w_j$.
The IS estimate of $\mathbb{E}_{p(\theta|y,\xi)}[f(\theta)]$ is
\begin{equation}
{\mathbb{E}}[f(\theta)] = \sum_{i=1}^{N} \tilde{w}_i\, f(\theta^{(i)}).
\label{eq:is-estimator}
\end{equation}
This estimator is reused throughout
Appendices~\ref{app:siqr} and~\ref{app:pk}.

\subsection{Risk functionals: tractable implementations}
\label{app:risk-impl}

\paragraph{Chance-constraint scenario selection.}
\label{app:cc-saa}
Select $S(y,\xi) \subseteq \{1,\ldots,N\}$ as the smallest set
(ordered by descending $\tilde{w}_i$) satisfying
$\sum_{i \in S} \tilde{w}_i \geq \eta$.
The chance-constrained program (Eq.~7) enforces
$c_j(\vect g,\theta^{(i)}) \le 0$ for all $i \in S(y,\xi)$.

\paragraph{CVaR epigraph reformulation.}
\label{app:cvar-saa}
Introducing slack $s_i \ge 0$ and threshold $\tau_j$,
$\mathrm{CVaR}_\eta(c_j) \le 0$ becomes:
\begin{equation}
\tau_j + \frac{1}{1-\eta}\sum_{i=1}^{N}\tilde{w}_i s_i \le 0,
\quad s_i \ge c_j(\vect g,\theta^{(i)}) - \tau_j, \quad s_i \ge 0.
\label{eq:cvar-epigraph}
\end{equation}
This form is used in both Appendix~\ref{app:siqr} and~\ref{app:pk}.

\subsection{Envelope theorem and implicit differentiation}
\label{app:envelope}
For a parametric convex program with Lagrange multipliers $\lambda^*$ at optimality:
\[
\nabla_\theta J^\ast(y,\xi) = \lambda^{*\top} \nabla_\theta c(\vect g^*, \theta).
\]
Jacobians are obtained via implicit differentiation of the KKT conditions
using \texttt{cvxpylayers}~\citep{agrawal2019differentiable}.
Used in Appendices~\ref{app:siqr},~\ref{app:pk}, and~\ref{app:structural:proof}.

\section{Algorithmic and inference details}
\label{app:alg}
\subsection{GoBOED algorithm (pseudocode)}
\label{app:goboed-alg}

Algorithm~\ref{alg:goboed} summarizes the two-stage GoBOED procedure:
offline training of the amortized variational posterior, followed by
gradient-based optimization of the design $\xi$ under the robust
decision-making objective.

\begin{algorithm}[t]
\caption{GoBOED: Goal-driven Bayesian Optimal Experimental Design}
\label{alg:goboed}
\footnotesize
\begin{algorithmic}[1]
\REQUIRE Prior $p(\btheta)$, design space $\Xi$, forward model $p(y \mid \btheta,\xi)$,
         robust control problem (cost $J(g)$, constraints $c_j(g,\theta)$),
         amortized VI network $q_\phi(\btheta \mid y,\xi)$, step sizes $\omega_\phi,\omega_\xi$,
         Monte Carlo batch size $M$ for design optimization
\ENSURE Approximate optimal design $\xi^*$ and variational parameters $\phi^*$

\STATE \textbf{Offline: train amortized variational inference}
\STATE Initialize variational parameters $\phi$
\WHILE{not converged}
  \STATE Sample $\xi \sim p(\xi)$, $\btheta \sim p(\btheta)$, $y \sim p(y \mid \btheta,\xi)$
  \STATE Compute ELBO $\mathcal{L}_{\text{VI}}(\phi; y,\xi)$ as in~\eqref{eq:SVI}
  \STATE Update $\phi \gets \phi + \omega_\phi \,\nabla_\phi \mathcal{L}_{\text{VI}}(\phi; y,\xi)$
\ENDWHILE

\STATE \textbf{Design optimization: decision-focused loop}
\STATE Define design loss $L(\xi) = \mathbb{E}_{p(y\mid\xi)}\!\big[J^*(y,\xi)\big]$
\STATE Initialize design $\xi$ (e.g.\ mid-point of allowed time window)
\WHILE{not converged}
  \STATE Set Monte Carlo estimate $\hat L(\xi) \gets 0$
  \FOR{$m = 1,\dots,M$}
    \STATE Sample $y^{(m)} \sim p(y \mid \xi)$
    \STATE Build posterior $q_\phi(\btheta \mid y^{(m)},\xi)$ and draw $\btheta$ via reparameterization
    \STATE Solve the robust decision-making problem under $q_\phi(\btheta \mid y^{(m)},\xi)$
           via the decision layer, obtaining $J^*(y^{(m)},\xi)$
    \STATE $\hat L(\xi) \gets \hat L(\xi) + J^*( y^{(m)},\xi)$
  \ENDFOR
  \STATE $\hat L(\xi) \gets \hat L(\xi)/M$
  \STATE Compute stochastic gradient $g_\xi \gets \nabla_\xi \hat L(\xi)$
         using reparameterization + implicit differentiation
  \STATE Update $\xi \gets \xi - \omega_\xi \, g_\xi$
\ENDWHILE
\STATE \textbf{return} $\xi^* \gets \xi$, $\phi^* \gets \phi$
\end{algorithmic}
\end{algorithm}

\subsection{Variational network architecture}
\label{app:NN-arch}
Let $\xi\in\Xi$ denote the design and $y\in\mathcal{Y}$ the observation.
Both are mapped into a shared latent space using linear tokenizers that
produce $M$ tokens of width $d$:
\[
E_\xi:\mathbb{R}^{D_\xi}\!\to\mathbb{R}^{M\times d},\quad
E_y:\mathbb{R}^{D_y}\!\to\mathbb{R}^{M\times d},\qquad
Z_\xi=E_\xi(\xi),\; Z_y=E_y(y),
\]
where in PK we set $D_\xi=1$, $D_y=1$, $M=8$, and $d=64$.
Queries and keys are linear projections of the token sequences,
\[
q_j=W_q\,z_{y,j}\in\mathbb{R}^d,\qquad
k_i=W_k\,z_{\xi,i}\in\mathbb{R}^d,
\]
and values fuse per-token key and query latents via a small MLP,
\[
v_i=\Psi\!\bigl([\,z_{\xi,i},\,z_{y,i}\,]\bigr)\in\mathbb{R}^d,\qquad i,j\in\{1,\dots,M\}.
\]
Single-head dot-product cross-attention over tokens is
\[
a_{ji}=\frac{\exp\!\bigl(q_j^\top k_i/\sqrt{d}\bigr)}{\sum_{i'=1}^M \exp\!\bigl(q_j^\top k_{i'}/\sqrt{d}\bigr)},
\qquad
c_j=\sum_{i=1}^{M} a_{ji}\,v_i\in\mathbb{R}^d.
\]
We mean-pool the query contexts and pass through a light trunk MLP:
\[
s=\frac{1}{M}\sum_{j=1}^{M} c_j,\qquad
h=\mathrm{MLP}_{\mathrm{trunk}}(s)\in\mathbb{R}^d.
\]
Two heads with skip connections produce log-space parameters for a diagonal LogNormal posterior,
\[
\begin{aligned}
\tilde{\mu}    &= W_{\ell,2}\,\phi(W_{\ell,1}h) + W_{\ell,\mathrm{skip}}h,\\
\tilde{\sigma} &= W_{s,2}\,\phi(W_{s,1}h) + W_{s,\mathrm{skip}}h,
\end{aligned}
\]
which we bound elementwise:
\[
\mu    \;=\; \mu_0 \;+\; \Delta_{\max}\,\tanh(\tilde{\mu}),\qquad
\sigma \;=\; \sigma_{\min} \;+\; (\sigma_{\max}-\sigma_{\min})\,\boldsymbol{\sigma}\!\bigl(\tilde{\sigma}\bigr).
\]
Here $\phi$ is GELU, $\mu_0$ is the prior log-mean (used as a residual center),
$\Delta_{\max}$ bounds deviations, $0<\sigma_{\min}<\sigma_{\max}$ bound the
log-space standard deviations, and $\boldsymbol{\sigma}$ is sigmoid activation function.

\subsection{Optimization and solvers}

We model the semidefinite programs arising in the decision layer and
solve them using SCS~\citep{scs} with default settings. For comparison
and smaller instances, we also solved the problems with the
interior-point solver MOSEK~\citep{mosek} and obtained numerically
indistinguishable solutions.

\subsection{Implementation and training configuration}
\label{app:training}

We train all amortized variational posteriors $q_\phi(\cdot \mid y,\xi)$
offline, following Algorithm~\ref{alg:goboed}. For each experiment we
sample design--observation pairs $(\xi,y)$ from the prior predictive
model and optimize the ELBO in \cref{eq:SVI} using stochastic gradient
descent.

We use the AdamW optimizer \citep{loshchilov2017decoupled} for all experiments, with learning rates and Monte Carlo sample budgets tuned per task. The designs $\xi$ are chosen on a fixed grid over the design space (7$\times$7 grid points for source localization, 15 early observation days for the SIQR experiment, and 24 points for the PK experiment). \Cref{tab:training-config} summarizes the training configuration for each experiment. Gradients are computed via AD through the forward models and the decision layer.
\begin{table}[h]
\centering
\caption{Per-experiment training configuration for the amortized variational posterior.}
\label{tab:training-config}
\begin{tabular}{lcccc}
\toprule
Experiment & Learning rate & Outer samples & Inner samples & Epochs \\
\midrule
Source localization & $3\times 10^{-3}$ & $256$  & $16$ & $1,000$\\
SIQR                & $3\times 10^{-4}$ & $128$           & $128$ & $1,000$ \\
PK                  & $1\times 10^{-3}$ & $1{,}000$      & $400$ & $1,000$\\
\bottomrule
\end{tabular}
\end{table}

For source localization we use mini-batches of $256$ $(\xi,y)$ pairs per gradient step with $16$ reparameterized posterior samples. For SIQR we use $128$ outer samples over $y$ and $128$ inner posterior samples per outer sample. For PK we use $1{,}000$ outer samples and $400$ inner posterior samples drawn from $q_\phi(\cdot \mid y,\xi)$. We use $1,000$ epochs for all problems.

\subsection{Evaluation protocol and Monte Carlo budgets}
\label{app:mc-budgets}

\begin{table}[h]
\centering
\caption{Monte Carlo sampling budgets per experiment.}
\label{tab:mc-budgets}
\begin{tabular}{llrr}
\toprule
Experiment & Metric & Outer $y$ & Inner/posterior $\theta$ \\
\midrule
Source localization & EIG \& MSE & 10{,}000 & 10{,}000 \\
SIQR               & EIG        & 5{,}000  & 3{,}000  \\
PK                 & EIG        & 5{,}000  & 3{,}000  \\
SIQR               & Robust control & 500  & 40       \\
PK                 & Robust control & 500  & 40      \\
\bottomrule
\end{tabular}
\end{table}

\subsection{Compute resources}
\label{app:compute}
All experiments were run on a single NVIDIA A100 GPU (40\,GB HBM2).
Approximate wall-clock times: source localization $\approx 30$\,min,
SIQR $\approx 60$\,min, PK $\approx 45$\,min (Stage~1 training).
Stage 2 design optimization required approximately 10 min for PK and 20–35 min for SIQR, depending on the risk functional and number of gradient steps.

\section{Model and control formulations}
\label{app:models}
\subsection{Source localization toy model}
We consider a two-dimensional source-localization problem in which a
single sensor must be placed to localize an emitting point source. The unknown
parameter $\btheta \in \mathbb{R}^{2}$ encodes the Cartesian coordinates of the
source, written as $\btheta = (\btheta_x,\btheta_y)^\top$. We place an
independent standard normal prior on each coordinate,
\[
\btheta \sim \mathcal{N}(0,I_{2}).
\]

The design variable is the sensor location $\xi \in \mathbb{R}^2$. For a given
design $\xi$ and parameter $\btheta$, the forward map returns a
three-dimensional summary of the local signal field. Let
\[
d(\xi,\btheta) = \btheta - \xi \in \mathbb{R}^2,\qquad
r^2(\xi,\btheta) = \|d(\xi,\btheta)\|_2^2
\]
denote the offset and squared distance from the sensor to the source. We assign
an inverse-square-type weight
\[
w(\xi,\btheta) = \frac{1}{s_0 + r^2(\xi,\btheta)},
\]
with $s_0>0$ a small regularization constant, and define the total intensity
\[
I(\xi,\btheta)
=
c_{\mathrm{base}} + w(\xi,\btheta),
\]
where $c_{\mathrm{base}}$ is a background level.

We also summarize the bearing from the sensor towards the source. Since there is
only one source, this is simply the direction of $d(\xi,\btheta)$. Let
\[
\bar\psi(\xi,\btheta) = \mathrm{atan2}\bigl(d_y(\xi,\btheta),d_x(\xi,\btheta)\bigr)
\]
denote the polar angle of the offset. To avoid angular discontinuities we
encode this angle via $(\cos\bar\psi,\sin\bar\psi)$. The noiseless observation
is therefore
\[
m(\xi,\btheta)
=
\bigl(\log I(\xi,\btheta),\ \cos\bar\psi(\xi,\btheta),\ \sin\bar\psi(\xi,\btheta)\bigr)
\in\mathbb{R}^3,
\]
and we model the noisy measurement as
\[
y \mid \btheta,\xi
\sim
\mathcal{N}\!\bigl(m(\xi,\btheta), \,
\mathrm{diag}(\sigma_I^2,\sigma_\phi^2,\sigma_\phi^2)\bigr),
\]
with different noise scales for intensity and angular components.

For this toy problem we treat the vertical source coordinate $\btheta_y$ as the
decision-relevant quantity. Given an observation $(\xi,y)$, the downstream
“decision’’ is the scalar estimate $\hat\btheta_y(\xi,y)$ produced by the
amortized posterior network $q_\phi(\btheta\mid\xi,y)$. We use the squared
error in $\btheta_y$ as the loss,
\[
\mathrm{MSE}(\xi,y;\btheta)
=
\bigl(\hat\btheta_y(\xi,y)-\btheta_y\bigr)^2,
\]
and the decision-focused design objective is the expected posterior mean squared
error
\[
J(\xi)
=
\mathbb{E}_{\btheta,y}\bigl[\mathrm{MSE}(\xi,y;\btheta)\bigr],
\]
where the expectation is taken over the prior $p(\btheta)$ and the likelihood
$p(y\mid\btheta,\xi)$.

\subsection{Optimal control for epidemiology}
\label{sec:back-ocforepi}
Building on the BOED framework, we consider robust epidemiology control as an example. Our method leverages the framework established by \cite{talaei2024network}. The epidemiology model is governed by a SIQR spread disease network, where state variables represent different compartments of the population: susceptible ($s$), asymptomatic infected ($x^a$), symptomatic infected ($x^s$) and recovered ($h$). 
\begin{equation}
\begin{pmatrix}
\dot{s} \\
\dot{x}^{a} \\
\dot{x}^{s} \\
\dot{h}
\end{pmatrix} =
\begin{pmatrix}
0 & -\beta^{a}s & -\beta^{s}s & 0 \\
0 & \beta^{a}s - \epsilon - \gamma^{a} - g^a & \beta^{s}s & 0 \\
0 & \epsilon  & -\gamma^{s} - g^s & 0 \\
0 & \gamma^{a} & \gamma^{s} & 0
\end{pmatrix}
\begin{pmatrix}
s \\
x^{a} \\
x^{s} \\
h
\end{pmatrix}.
\end{equation}
Here, $\beta^a$ and $\beta^s$ are transmission rates for asymptomatic and symptomatic cases, $\epsilon$ is the rate at which asymptomatic cases develop symptoms, $\gamma^a$ and $\gamma^s$ are recovery rates for asymptomatic and symptomatic cases, and $g^a$ and $g^s$ are quarantine rates for asymptomatic and symptomatic individuals.

Following \cite{ma2023optimal}, we decouple the dynamics of $\dot{x}$ from $\dot{s}$ and $\dot{h}$, allowing us to focus on the matrix $M(t_0)$ which captures the essential infection dynamics at initial time $t_0$:
\begin{equation}
M(t_0) =
\begin{pmatrix}
\beta^{a}s(t_0) - \epsilon - \gamma^{a} - g^a & \beta^{s}s(t_0) \\
\epsilon & -\gamma^{s} -g^s
\end{pmatrix}\label{eq:M}
\end{equation}

To optimize the quarantine strategy as described in \cite{talaei2024network}, we utilize an objective function that minimizes economic costs:
\begin{equation}
\min_{g^a, g^s} J(g^a,g^s) = \frac{z^a}{1-g^a} + \frac{z^s}{1-g^s} \label{eq:siqr-constraints-main}
\end{equation}
\begin{align}
\text{s.t. }\lambda_{\text{max}}(M(t_0)) &\leq -\alpha, \\
0 \leq g^a & < 1, \\
0 \leq g^s & <  1,
\end{align}
where $z^a$ represents the economic cost for asymptomatic quarantine, $z^s$ represents the economic cost for symptomatic quarantine, and $\alpha>0$ is a constraint ensuring the stability of the system. 

The detailed solution for this minimization problem is provided in \Cref{app:siqr-gradients}. This approach allows us to determine optimal quarantine rates without explicitly integrating the SIQR differential equations. Instead, by analyzing the eigenvalues of the system and applying convex optimization, we can efficiently identify the optimal quarantine strategy that minimizes economic costs.

We follow the parameterization framework of
\citet{talaei2024network}. We place independent log-normal priors on
the transmission and recovery rates (in units of counts per day). In
log-space, the transmission rates for asymptomatic and symptomatic individuals, $\beta^a$ and $\beta^s$, and the recovery rates
$\gamma^a$ and $\gamma^s$ are log-normally distributed as
\[
(\log \beta^a, \log \beta^s, \log \gamma^a, \log \gamma^s)
\sim \mathcal{N}\bigl((0.5,\,0.8,\,0.2,\,0.2),
                      \mathrm{diag}(0.5^2,0.5^2,0.3^2,0.3^2)\bigr).
\]
The stability margin is fixed at $\alpha = 0.05$, and the economic
cost parameters are set to $(z^a,z^s) = (0.4,\,0.6)$.

The design variable $\xi \in [1,100]$ denotes the observation day. In the numerical experiments, we evaluate candidate observation days on the restricted grid $\xi \in \{1,\ldots,15\}$, corresponding to the early observation window considered in \Cref{fig:combined-fig}. At
time $\xi$ we observe
$y_{\text{obs}} = (y_{\text{obs}}^{a},y_{\text{obs}}^{s})$, the counts
of asymptomatic and symptomatic infected individuals. We model these
observations with a Poisson likelihood whose rate is a slightly
perturbed version of the model prediction,
\[
y_{\text{obs}} \mid \xi
\sim \mathrm{Poisson}\bigl(\lambda(\xi)\bigr),
\qquad
\lambda(\xi) = 0.95 \, y_{\text{true}}(\xi),
\]
where $y_{\text{true}}(\xi)$ is the SIQR solution evaluated at time
$\xi$ under a given parameter draw. The amortized variational
posterior $q_\phi(\bbeta,\bgamma \mid y,\xi)$ is trained on simulated
pairs $(y,\xi)$ as described in \Cref{app:NN-arch}, and the resulting
posterior is used inside the robust control layer. The convex
optimization problems for quarantine rates $(g^a,g^s)$ are solved via a semidefinite programming formulation
(see \Cref{app:siqr-gradients} for details).

\subsection{Optimal control for Pharmacokinetic model}
\label{sec:back-ocforpk}

We further consider PK model as another example. 
The concentration at time $t$ for a drug administered orally can be modeled using the Bateman function \citep{bateman1910solution}: 
\[
y(t) = \frac{D \cdot k_a}{V \cdot (k_a - k_e)} \left( e^{-k_e \cdot t} - e^{-k_a \cdot t} \right)(1 + \epsilon_\text{mult}) + \epsilon_\text{add},
\]
where $V$ is the volume of distribution, $k_a$ is the absorption rate constant, $k_e$ is the elimination rate constant, $D$ is the dose administered, $\epsilon_\text{mult}$ is a multiplicative error term, and $\epsilon_\text{add}$ is an additive error term. This formulation captures the dynamics of drug absorption and elimination. 

Given the drug's potential toxicity, dosing should maintain systemic exposure within the therapeutic window, avoiding toxic concentrations while not falling below the minimum effective concentration. The time at which the maximum drug concentration occurs, denoted $t_{\max}$, can be found by setting $\partial y / \partial t = 0$, which yields:
\[
t_{\max} = \frac{\ln(k_a / k_e)}{k_a - k_e}.
\] 

The maximum concentration, $C_{\max}$, is obtained by evaluating $y(t)$ at $t_{\max}$:
\[
C_{\max} = \frac{D}{V} \left( \frac{k_e}{k_a} \right)^{ \frac{k_e}{-k_e + k_a} } (1 + \epsilon_\text{mult}) + \epsilon_\text{add}.
\]


For cumulative exposure, the area under the concentration curve (AUC) is given by:
\[
\text{AUC} = \int_0^\infty y(t) dt = \frac{D}{V \cdot k_e},
\]
assuming complete absorption.

We define a convex cost function $J(g)$ (e.g., $J(g)=c\,g$ to discourage high dosing). The constrained problem can be written in the following forms, 

\begin{align}
\label{eq:cc-pk-appendix}
\min_{\,0\le g\le 1}\quad & J(g) \\
\text{s.t.}\quad
& C_{\max}(g,\btheta)\le C_{\text{thresh}} , \nonumber\\
& \mathrm{AUC}(g,\btheta) \;\ge\; \mathrm{AUC}_{\min}. \nonumber
\end{align}

We place independent log-normal priors
on the absorption rate $k_a$, elimination rate $k_e$, and volume of
distribution $V$. In log-space, the parameters satisfy
\[
(\log k_a, \log k_e, \log V)
\sim \mathcal{N}\bigl((0,\,\log 0.1,\,\log 20.0),
                      \mathrm{diag}(0.2^2,0.2^2,0.2^2)\bigr).
\]
The Bateman PK model, the definitions of $C_{\max}$ and AUC, and the
associated robust control formulations (chance constraints and CVaR)
are given in \Cref{app:pk-risk}. The
same amortized variational inference and importance-weighted posterior
estimation pipeline is used as in the SIQR experiments, with the
control variable $g$ parameterizing the administered dose $D(g)$.

\section{Case Study 2: SIQR Epidemic Control}
\label{app:siqr}

\subsection{Model dynamics, prior, and observation model}
\label{app:siqr-model}

We develop an integrated framework, GoBOED, for epidemic management by
bringing together BOED and optimal
control through convex optimization (described in \cref{eq:opt-theta}). Our
primary objective is to identify an experimental design
$\xi^*$ that minimizes the expected controlled economic cost, where
the expectation is taken over the posterior distributions of the
parameters $\bbeta = (\beta^a,\beta^s)$ and
$\bgamma = (\gamma^a,\gamma^s)$. This allows us to update our beliefs
with new observations and improve decision-making, following the
Bayesian decision-theoretic perspective of \citet{chaloner1995bayesian}.

For a given design $\xi$ and observed data $y$, we consider
\begin{equation}
\min_{g^a,g^s} \,
\mathbb{E}_{p(\bbeta,\bgamma \mid y,\xi)}
\bigl[ J(g^a,g^s) \bigr],
\label{eq:sim-ep}
\end{equation}
subject to the posterior-averaged stability constraint
\begin{equation}
\begin{aligned}
\lambda_{\max}\Bigl(
\mathbb{E}_{p(\bbeta,\bgamma \mid y,\xi)}
\bigl[ M(t_0,\bbeta,\bgamma,g^a,g^s) \bigr]
\Bigr)
&\le -\alpha, \\
0 \le g^a < 1,\quad
0 \le g^s < 1.
\end{aligned}
\label{eq:e-constraints}
\end{equation}
Here, $J(g^a,g^s)$ denotes the economic cost, which does not depend
explicitly on $(\bbeta,\bgamma)$; hence the objective in
\cref{eq:sim-ep} reduces to minimizing $J(g^a,g^s)$ subject to
\cref{eq:e-constraints}. The main difficulty lies in evaluating the
eigenvalue constraint, which involves the posterior expectation of the
matrix $M(t_0,\bbeta,\bgamma,g^a,g^s)$.

To approximate the posterior $p(\bbeta,\bgamma \mid y,\xi)$ we employ
the amortized variational inference scheme introduced in the main
text. In particular, we use the variational family
$q_\phi(\bbeta,\bgamma \mid y,\xi)$ and optimize $\phi$ via the ELBO in
\cref{eq:SVI}. Expectations under the true posterior are then
estimated using the importance-sampling estimator defined in
\cref{eq:is-theta} (see also Appendix~\ref{app:is-estimator}).
For the spectral stability constraint, we take
$f(\bbeta,\bgamma) = M(t_0,\bbeta,\bgamma,g^a,g^s)$ in that estimator,
which yields an approximation of
$\mathbb{E}_{p(\bbeta,\bgamma \mid y,\xi)}[M(t_0,\bbeta,\bgamma,g^a,g^s)]$
and thereby of the eigenvalue constraint in \cref{eq:e-constraints}.

\subsection{Robust quarantine control formulation}
\label{app:siqr-risk}
We instantiate the general risk functionals of
Appendix~\ref{app:risk-impl} via the per-sample stability violation:
\[
v_i(g_a,g_s) = \lambda_{\max}(M(t_0,\beta_i,\gamma_i,g_a,g_s)) + \alpha.
\]

\paragraph{Mean-based.}
Enforce $c(g,\bar\theta) = \lambda_{\max}(M(t_0,\bar\beta,\bar\gamma,g_a,g_s)) + \alpha \le 0$
at the posterior mean $\bar\theta$ (see Appendix~\ref{app:is-estimator}).

\paragraph{Chance-constrained.}
\label{app:siqr-cc}
Use scenario selection of Appendix~\ref{app:cc-saa} with
$f(\theta) = \mathbf{1}[\lambda_{\max}(\cdot) > -\alpha]$.
\begin{equation}
\min_{g^a,g^s} J(g^a,g^s),
\label{eq:simp-form}
\end{equation}
subject to
\begin{align}
\mathbb{P}\!\left(
\lambda_{\max}\bigl(M(t_0,\bbeta,\bgamma,g^a,g^s)\bigr)
\le -\alpha \,\middle|\, y,\xi
\right)
&\ge \eta,
\label{eq:chance-constraint} \\
0 \le g^a < 1, \\
0 \le g^s < 1,
\end{align}
where $(\bbeta,\bgamma)$ are distributed according to
$p(\bbeta,\bgamma \mid y,\xi)$ and $\eta \in (0,1)$ is the desired
confidence level. Plugging
$f(\bbeta,\bgamma) = \mathbf{1}\{\lambda_{\max}(M)\le -\alpha\}$ into the
self-normalized IS estimator (Appendix~\ref{app:is-estimator}) yields
\[
\mathbb{P}\!\left(
\lambda_{\max}\bigl(M(t_0,\bbeta,\bgamma,g^a,g^s)\bigr)
\le -\alpha \,\middle|\, y,\xi
\right)
\approx
\sum_{i=1}^N \tilde{w}_i
\,\mathbf{1}\!\left\{
\lambda_{\max}\bigl(M(t_0,\bbeta_i,\bgamma_i,g^a,g^s)\bigr)
\le -\alpha
\right\}.
\]
Following \cref{eq:cc-scenario}, the posterior samples
$\{(\bbeta_i,\bgamma_i)\}_{i=1}^N\sim q_\phi$ and their IS weights
$\{\tilde w_i\}_{i=1}^N$ are drawn once per $(y,\xi)$, and a scenario
set $\mathcal{S}(y,\xi)$ with $\sum_{i\in\mathcal{S}}\tilde w_i\ge\eta$
is selected by the rule of Appendix~\ref{app:cc-saa} and held fixed
during the inner optimization over $(g^a,g^s)$.

\paragraph{CVaR-based.}
\label{app:siqr-cvar}
Use the epigraph form of Appendix~\ref{app:cvar-saa} with $v_i$ above. As in the chance-constrained case, posterior samples
$\{(\bbeta_i,\bgamma_i)\}_{i=1}^N\sim q_\phi$ and their IS weights
$\{\tilde w_i\}_{i=1}^N$ are drawn once per $(y,\xi)$ and held fixed
during inner optimization.

With decision variables $\tau\in\mathbb{R}$ and $s_i \ge 0$. At
optimality, $s_i = (v_i(g^a,g^s)-\tau)_+$, so only the upper tail
beyond the $\eta$-quantile contributes. The robust controls are then
obtained by solving the convex program
\begin{equation}
\label{eq:cvar-problem-general-app}
\begin{aligned}
\min_{g^a,g^s,\ \tau,\ \{s_i\}} \quad
& J(g^a,g^s)\\
\text{s.t.}\quad
& s_i \ \ge\ v_i(g^a,g^s)-\tau,\quad s_i\ge 0,\quad i=1,\dots,N,\\
& \tau+\frac{1}{1-\eta}\sum_{i=1}^N \tilde w_i\, s_i \ \le\ 0,\\
& 0\le g^a < 1,\quad 0\le g^s < 1,
\end{aligned}
\end{equation}
where $J(g^a,g^s)$ is the same convex objective as in
\Cref{eq:simp-form}. The spectral violation $v_i$ admits a convex
epigraph, so \cref{eq:cvar-problem-general-app} remains a tractable
semidefinite program.

\subsection{SDP formulation, KKT conditions, and gradients}
\label{app:siqr-gradients}

We now derive gradients for the SIQR control problem in
\cref{eq:siqr-constraints-main} using a Lagrangian and KKT conditions.
Here $(\bbeta_i,\bgamma_i)$ denote the sampled transmission and
recovery rates for scenario $i$, as introduced in
\Cref{sec:back-ocforepi}.

To address this constrained optimization problem, we introduce the Lagrangian:
\begin{equation*}
\begin{aligned}
\mathcal{L} = \frac{z^a}{1 - g^a} + \frac{z^s}{1 - g^s}
&+ \sum_{i=1}^N \lambda_i(\lambda_{\text{max}}(M(t_0,\bbeta_i,\bgamma_i, g^a, g^s)) + \alpha)
\\&- \mu_1 g^a + \mu_2(g^a - 1) - \mu_3 g^s + \mu_4(g^s - 1),
\end{aligned}
\end{equation*}
where $\lambda_i$ are Lagrange multipliers associated with the eigenvalue constraints, and $\mu_i$ are multipliers for the box constraints on $g^a$ and $g^s$. The optimal solution must satisfy the Karush--Kuhn--Tucker (KKT) conditions, which we outline below.

\subsubsection{KKT conditions}

\paragraph{Stationarity}
The stationarity conditions are derived by taking partial derivatives of the Lagrangian:
\begin{align*}
\frac{\partial \mathcal{L}}{\partial g^a} &= \frac{z^a}{(1 - g^a)^2} + \sum_{i=1}^N \lambda_i \frac{\partial \lambda_{\text{max}}(M(t_0, \bbeta_i, \bgamma_i, g^a, g^s))}{\partial g^a} + \mu_2 - \mu_1 = 0, \\
\frac{\partial \mathcal{L}}{\partial g^s} &= \frac{z^s}{(1 - g^s)^2} + \sum_{i=1}^N \lambda_i \frac{\partial \lambda_{\text{max}}(M(t_0, \bbeta_i, \bgamma_i, g^a, g^s))}{\partial g^s} + \mu_4 - \mu_3 = 0.
\end{align*}

\paragraph{Primal feasibility}
The primal feasibility conditions ensure the constraints hold:
\begin{align*}
\lambda_{\text{max}}(M(t_0, \bbeta_i, \bgamma_i, g^a, g^s)) &\leq -\alpha \quad \text{for} \quad i = 1, \ldots, N, \\
0 \leq g^a < 1, \\
0 \leq g^s < 1.
\end{align*}

\paragraph{Dual feasibility}
The Lagrange multipliers must be non-negative:
\[
\lambda_i \geq 0 \quad \text{for} \quad i = 1, \ldots, N, \quad \text{and} \quad \mu_1, \mu_2, \mu_3, \mu_4 \geq 0.
\]

\paragraph{Complementary slackness}
The complementary slackness conditions are:
\begin{align*}
\lambda_i \left( \lambda_{\text{max}}(M(t_0, \bbeta_i, \bgamma_i, g^a, g^s)) + \alpha \right) &= 0 \quad \text{for} \quad i = 1, \ldots, N, \\
\mu_1 g^a &= 0, \\
\mu_2 (1 - g^a) &= 0, \\
\mu_3 g^s &= 0, \\
\mu_4 (1 - g^s) &= 0.
\end{align*}

Assuming an interior solution, the multipliers for the box constraints at the optimal point become \(\mu_i = 0\). The KKT conditions simplify to:
\[
\frac{z^a}{(1 - {g^a}^*)^2} + \sum_{i=1}^N \hat{\lambda}_i \frac{\partial \lambda_{\text{max}}(M(t_0, \bbeta_i, \bgamma_i, {g^a}^*, {g^s}^*))}{\partial g^a} = 0,
\]
\[
\frac{z^s}{(1 - {g^s}^*)^2} + \sum_{i=1}^N \hat{\lambda}_i \frac{\partial \lambda_{\text{max}}(M(t_0, \bbeta_i, \bgamma_i,{g^a}^*, {g^s}^*))}{\partial g^s} = 0,
\]
\[
\lambda_{\text{max}}(M(t_0, \bbeta_{i^*}, \bgamma_{i^*}, {g^a}^*, {g^s}^*)) = -\alpha \quad \text{for specific} \quad i^*,
\]
where $\hat{\lambda}_i$ denotes the optimal Lagrange multipliers.

We compute the optimal values ${g^a}^*$, ${g^s}^*$, and $\hat{\lambda}_{i^*}$ using convex optimization via semidefinite programming. This is feasible because \(J\) is a convex function with respect to \(g^a\) and \(g^s\), and the largest eigenvalue constraint can be reformulated as a set of linear matrix inequalities.

\subsubsection{Derivative of the optimum cost \(J^\ast\) with respect to model parameters}
By Appendix~\ref{app:envelope}, the gradient of the optimum cost \(J^\ast\) with respect to the model parameters \(\bbeta\) and \(\bgamma\) follows from the envelope identity. The derivative with respect to \(\beta_i^a\) is:
\[
\frac{\partial J^\ast}{\partial \beta_i^a}
= \hat{\lambda}_i \cdot \,
\frac{\partial \lambda_{\text{max}}\bigl(M(t_0, \bbeta_i, \bgamma_i, {g^a}^*, {g^s}^*)\bigr)}{\partial \beta_i^a}.
\]

Similarly, the derivatives with respect to other parameters are:
\[
\frac{\partial J^\ast}{\partial \gamma_i^s} = \hat{\lambda}_i \cdot \frac{\partial \lambda_{\text{max}}(M(t_0, \bbeta_i, \bgamma_i, {g^a}^*, {g^s}^*))}{\partial \gamma_i^s},
\]
\[
\frac{\partial J^\ast}{\partial \gamma_i^a} = \hat{\lambda}_i \cdot \frac{\partial \lambda_{\text{max}}(M(t_0, \bbeta_i, \bgamma_i, {g^a}^*, {g^s}^*))}{\partial \gamma_i^a},
\]
\[
\frac{\partial J^\ast}{\partial \beta_i^s} = \hat{\lambda}_i \cdot \frac{\partial \lambda_{\text{max}}(M(t_0, \bbeta_i, \bgamma_i, {g^a}^*, {g^s}^*))}{\partial \beta_i^s}.
\]

Since the optimal solution often lies on specific boundaries where \(\hat{\lambda}_{i^*} \neq 0\) and \(\hat{\lambda}_i = 0\) for \(i \neq i^*\), the gradient depends only on the samples directly influencing the solution.

\subsubsection{Derivative of the constraints with respect to experimental design}

The constraint term is defined as:
\[
l(\bbeta, \bgamma, g^a, g^s ;y, \xi) = \sum_{i=1}^N \lambda_i \left( \lambda_{\text{max}}(M(t_0, \bbeta_i, \bgamma_i, {g^a}, {g^s})) + \alpha \right) - \mu_1 g^a + \mu_2(g^a - 1) - \mu_3 g^s + \mu_4(g^s - 1).
\]

Its gradient with respect to \(\xi\) at the optimal point is:
\begin{equation}
    \begin{aligned}
\nabla_{\xi} l = \sum_{i=1}^N \hat{\lambda}_i \left( \frac{\partial \lambda_{\text{max}}}{\partial \beta_i^a} \cdot \nabla_{\xi} \beta_i^a + \frac{\partial \lambda_{\text{max}}}{\partial \beta_i^s} \cdot \nabla_{\xi} \beta_i^s + \frac{\partial \lambda_{\text{max}}}{\partial \gamma_i^a} \cdot \nabla_{\xi} \gamma_i^a + \frac{\partial \lambda_{\text{max}}}{\partial \gamma_i^s} \cdot \nabla_{\xi} \gamma_i^s \right).
    \end{aligned}
\end{equation}

The partial derivative \(\frac{\partial \lambda_{\text{max}}}{\partial \beta_i^a}\) is given by \citet{talaei2024network} as:
\[
\frac{\partial \lambda_{\text{max}}}{\partial \beta_i^a} = \frac{v_{\text{max}}^T \left( \frac{\partial M(t_0, \bbeta_i)}{\partial \beta_i^a} \right) u_{\text{max}}}{v_{\text{max}}^T u_{\text{max}}},
\]
where \(v_{\text{max}}\) and \(u_{\text{max}}\) are the left and right eigenvectors of the largest eigenvalue, respectively. The terms \(\nabla_{\xi} \beta_i^a\), \(\nabla_{\xi} \beta_i^s\), \(\nabla_{\xi} \gamma_i^a\), and \(\nabla_{\xi} \gamma_i^s\) are computed via automatic differentiation from the variational network.

\subsubsection{Marginal likelihood gradient}
\label{sec:marlikgrad}
To compute the log-likelihood gradient with respect to \(\xi\), we assume a Poisson observation model (suitable for count data) with rate parameter $\lambda = 0.95 \cdot y_{\text{true}}(\xi)$. The gradient is:
\[
\frac{\partial}{\partial \xi} \log p(y_{\text{obs}} \mid \xi) = \left( \frac{y_{\text{obs}}}{0.95 \cdot y_{\text{true}}(\xi)} - 1 \right) \cdot 0.95 \cdot \frac{\partial y_{\text{true}}(\xi)}{\partial \xi}.
\]

This expression helps quantify how changes in \(\xi\) affect the likelihood of the observed data. The term \(\frac{\partial y_{\text{true}}(\xi)}{\partial \xi}\) can be approximated using finite difference methods, such as the central difference method.

\section{Case Study 3: Pharmacokinetic Control}
\label{app:pk}

\subsection{Bateman model, $C_{\max}$, AUC, and prior}
\label{app:pk-model}

We further consider the PK model as another example.
The concentration at time $t$ for a drug administered orally can be modeled using the Bateman function \citep{bateman1910solution}:
\[
y(t) = \frac{D \cdot k_a}{V \cdot (k_a - k_e)} \left( e^{-k_e \cdot t} - e^{-k_a \cdot t} \right)(1 + \epsilon_\text{mult}) + \epsilon_\text{add},
\]
where $V$ is the volume of distribution, $k_a$ is the absorption rate constant, $k_e$ is the elimination rate constant, $D$ is the dose administered, $\epsilon_\text{mult}$ is a multiplicative error term, and $\epsilon_\text{add}$ is an additive error term. This formulation captures the dynamics of drug absorption and elimination.

Given the drug's potential toxicity, dosing should maintain systemic exposure within the therapeutic window---avoiding toxic concentrations while not falling below the minimum effective concentration. The time at which the maximum drug concentration occurs, denoted $t_{\max}$, can be found by setting the derivative $\partial y / \partial t = 0$, which yields:
\[
t_{\max} = \frac{\ln(k_a / k_e)}{k_a - k_e}.
\]

The maximum concentration, $C_{\max}$, is obtained by evaluating $y(t)$ at $t_{\max}$:
\[
C_{\max} = \frac{D}{V} \left( \frac{k_e}{k_a} \right)^{ \frac{k_e}{-k_e + k_a} } (1 + \epsilon_\text{mult}) + \epsilon_\text{add}.
\]

For cumulative exposure, the area under the concentration curve (AUC) is given by:
\[
\text{AUC} = \int_0^\infty y(t)\,dt = \frac{D}{V \cdot k_e},
\]
assuming complete absorption.

We define a convex cost function $J(g)$ (e.g., $J(g)=c\,g$ to discourage high dosing). The constrained problem can be written in the following form,
\begin{align}
\label{eq:cc-pk-ref}
\min_{\,0\le g\le 1}\quad & J(g) \\
\text{s.t.}\quad
& C_{\max}(g,\btheta)\le C_{\text{thresh}} , \nonumber\\
& \mathrm{AUC}(g,\btheta) \;\ge\; \mathrm{AUC}_{\min}. \nonumber
\end{align}

We follow the parameterization in
\citet{kleinegesse2021gradient}. We place independent log-normal priors
on the absorption rate $k_a$, elimination rate $k_e$, and volume of
distribution $V$. In log-space, the parameters satisfy
\[
(\log k_a, \log k_e, \log V)
\sim \mathcal{N}\bigl((0,\,\log 0.1,\,\log 20.0),
                      \mathrm{diag}(0.2^2,0.2^2,0.2^2)\bigr).
\]
The same amortized variational inference and importance-weighted
posterior estimation pipeline is used as in the SIQR experiments
(Appendix~\ref{app:siqr}), with the control variable $g$ parameterizing
the administered dose $D(g)$.

\paragraph{Parameter influence.}
For the PK model, the peak concentration $C_{\max}$ scales roughly as
$D/V$ and increases with $k_a$ (faster absorption, smaller
$t_{\max}$), decreases with $k_e$ (faster elimination), and is
sensitive to the ratio $k_a/k_e$ (flip--flop kinetics when $k_a<k_e$).

We follow the same design--control split as in the SIQR example. The
design $\xi$ (e.g., blood sampling schedule and/or formulation choice)
is used to learn the PK parameters
$\btheta=(k_a,k_e,V)$ via the posterior $p(\btheta \mid y,\xi)$. The
control variable $g\in[0,1]$ is a dose fraction, with the administered
dose given by $D(g)$ (e.g., $D(g)=gD_0$ or a more general convex
mapping). All exposure and risk quantities below depend on $g$ through
$D(g)$ and on $\btheta$. Posterior expectations in the PK setting are
computed using the self-normalized IS estimator from 
\cref{eq:is-theta} (see also Appendix~\ref{app:is-estimator}), now with
$\btheta$ in place of $(\bbeta,\bgamma)$.

\subsection{Robust dose-control formulation}
\label{app:pk-risk}
We instantiate Appendix~\ref{app:risk-impl} via:
\[
v_i(g) = C_{\max}(g,\theta^{(i)}) - C_{\mathrm{thresh}},
\]
with an AUC exposure constraint
$\widehat{\mathbb{E}}[\mathrm{AUC}(g,\theta)] \ge \mathrm{AUC}_{\min}$
(Appendix~\ref{app:is-estimator}).

\paragraph{Chance-constrained.}
\label{app:pk-cc}
\begin{align}
\label{eq:cc}
\min_{\,0\le g\le 1}\quad & J(g) \\
\text{s.t.}\quad
& \mathbb{P}\!\left(
    C_{\max}(g,\btheta) \le C_{\text{thresh}}
    \,\middle|\, y,\xi
  \right) \;\ge\; \eta, \nonumber\\
& \mathbb{E}\!\left[ \mathrm{AUC}(g,\btheta) \right]
  \;\ge\; \mathrm{AUC}_{\min}. \nonumber
\end{align}
Both the probability and the expectation are evaluated using
\cref{eq:is-theta}, with
$f(\btheta)=\mathbf{1}\{C_{\max}(g,\btheta)\le C_{\text{thresh}}\}$ and
$f(\btheta)=\mathrm{AUC}(g,\btheta)$, respectively (Appendix~\ref{app:cc-saa}).

\paragraph{CVaR-based.}
\label{app:pk-cvar}
With the per-sample toxicity violation (optionally with a nonnegative
margin $\alpha$)
\begin{equation}
\label{eq:viol}
v_i(g)
= C_{\max}(g,\btheta_i) - C_{\text{thresh}} + \alpha,
\qquad i=1,\dots,N,
\end{equation}
and normalized importance weights $\tilde w_i$ from
\cref{eq:is-theta}, we enforce $\mathrm{CVaR}_\eta(g)\le 0$ via the
epigraph form of Appendix~\ref{app:cvar-saa}:
\begin{align}
\label{eq:cvar}
\min_{\,0\le g\le 1,\;\tau,\;s_i\ge 0}\quad & J(g) \\
\text{s.t.}\quad
& s_i \;\ge\; v_i(g)-\tau,\qquad i=1,\dots,N, \nonumber\\
& \tau \;+\; \frac{1}{1-\eta}\sum_{i=1}^{N} \tilde{w}_i\, s_i \;\le\; 0, \nonumber\\
& \mathbb{E}\!\left[\mathrm{AUC}(g,\btheta)\right] \;\ge\; \mathrm{AUC}_{\min}. \nonumber
\end{align}
Here the expectation in the AUC constraint is again evaluated using
\cref{eq:is-theta}. This yields a PK control policy that trades off dose
cost against a tail-robust toxicity constraint and a minimum-exposure
requirement.

\subsection{Gradient derivations}
\label{app:pk-gradients}
Gradients follow from Appendix~\ref{app:envelope} with
$\nabla_\theta C_{\max}$ and $\nabla_\theta \mathrm{AUC}$
computed via automatic differentiation through the Bateman model.

\section{Numerical details}
\label{sec:NS-details}

\subsection{Additional posterior visualization result}
\label{sec:additionalvis}

\begin{figure}[t]
    \centering
    \includegraphics[width=\linewidth]{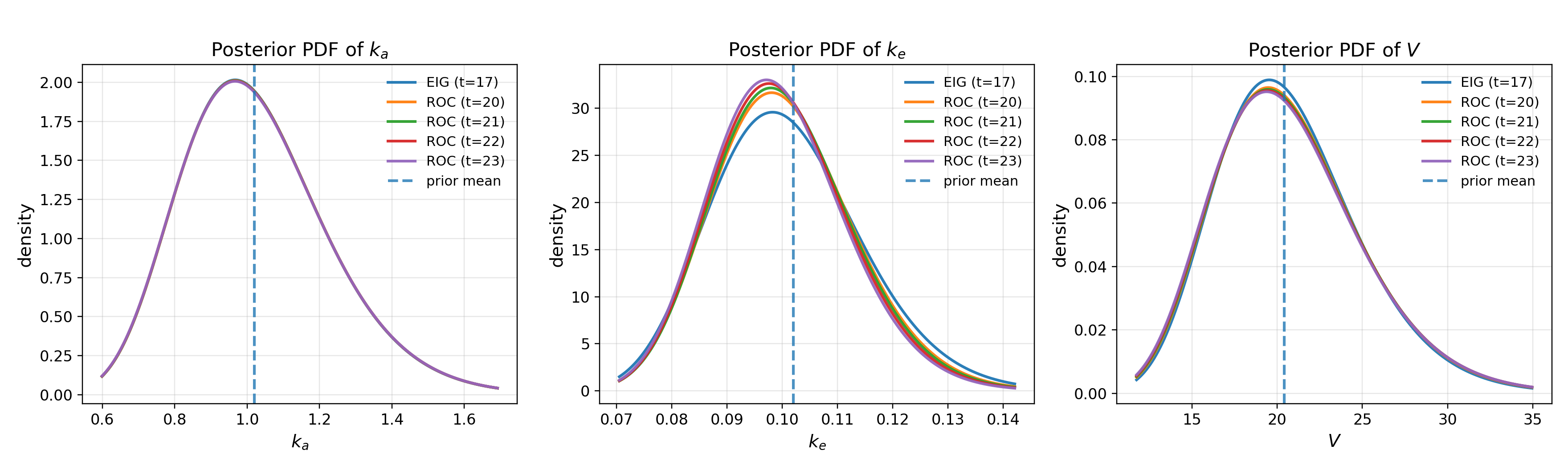}
    \caption{
    Posterior densities for the PK parameters $k_a$, $k_e$, and $V$ under different optimal designs.}
    \label{fig:pk-posteriors-full}
\end{figure}

\begin{figure}[t]
    \centering
    \includegraphics[width=\linewidth]{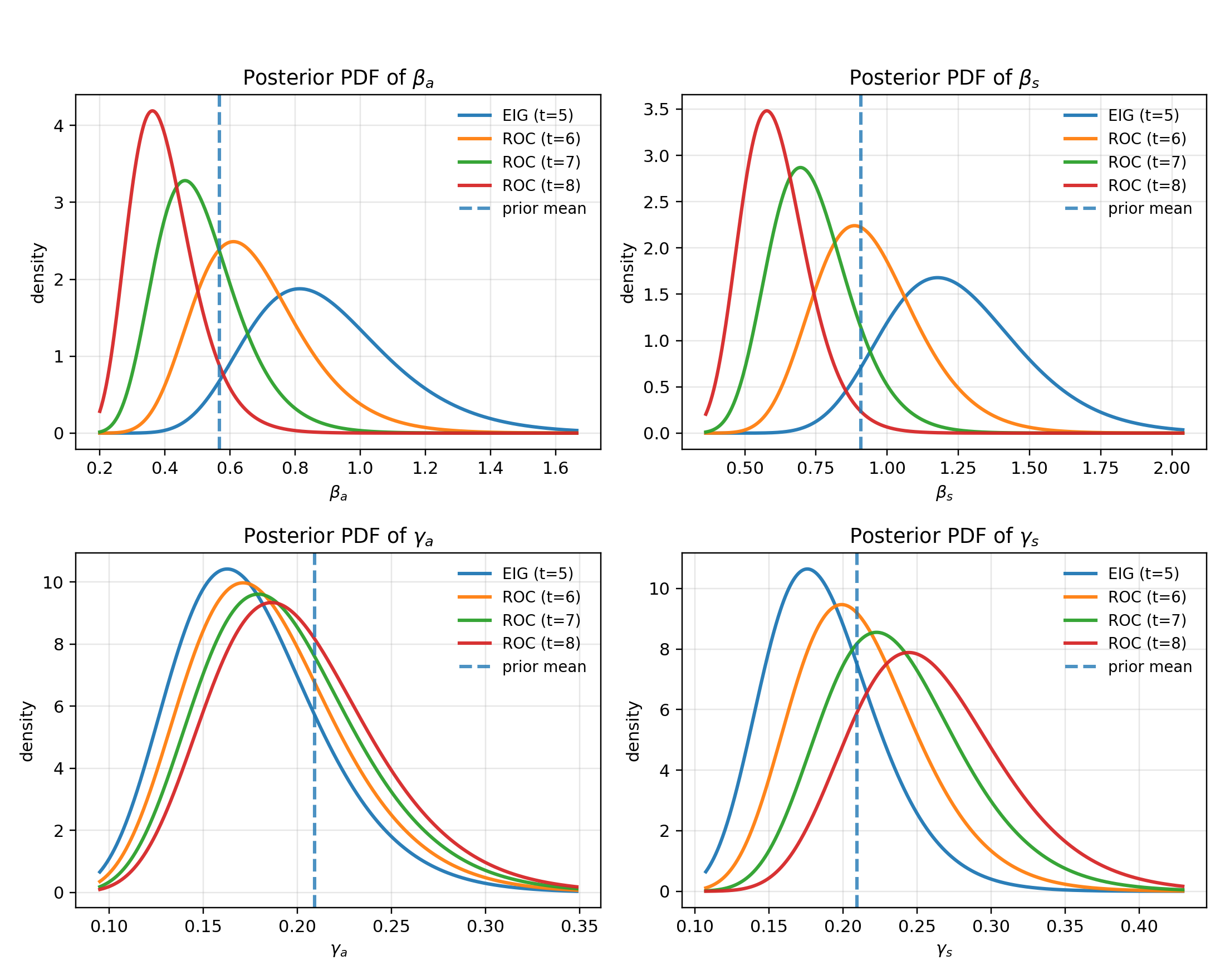}
    \caption{
    Posterior densities for the SIQR model parameters $(\beta_a,\beta_s,\gamma_a,\gamma_s)$ under different observation times.
    }
    \label{fig:siqr-posteriors}
\end{figure}

This appendix provides additional posterior visualizations for both the PK and SIQR models. In particular,
\Cref{fig:siqr-posteriors} shows the marginal posterior densities of the SIQR parameters under different designs.
The blue curves correspond to the EIG-optimal design at $t=5$, while the orange, green, and red curves correspond to representative GoBOED-related observation times near the robust-control favorable region. The dashed vertical
line denotes the prior mean, which is also used as the data-generating parameter in the forward model and therefore
serves as a common reference point for all posterior estimates. Compared to the EIG design, the GoBOED designs
produce posteriors that are more concentrated and, for the transmission and recovery rates, systematically shifted,
suggesting that GoBOED prioritizes parameter regions that most strongly affect the induced control policy.

\section{Extended Related Work}
\label{app:related_work}
\textbf{Goal-oriented Bayesian optimal experimental design} Recent advances in BOED have shifted focus from parameter estimation to optimizing for specific quantities of interest (QoIs), which are measurable outcomes that directly impact decision-making. For linear models, \citet{attia2018goal} established the framework for goal-oriented optimal design of experiments (GOODE) that simplifies computational evaluation for experimental design. Building on this work, \citet{neuberger2024goal} introduced a ``$G_q$-optimality'' criterion based on quadratic approximation of goal functionals for PDE-governed linear inverse problems. Additionally, for Bayesian linear inverse problems, \citet{madhavan2025control} developed a control-oriented approach that connects optimal control and sensor placement while prioritizing uncertainty reduction in controlled state variables. The linearity in these models makes the problems mathematically tractable and computationally efficient to solve. However, many real-world systems, including epidemic models, exhibit significant nonlinearities that require more sophisticated approaches.

For handling non-linear models, \citet{zhong2024goal} created a computational framework using nested Monte Carlo estimators, Markov chain Monte Carlo (MCMC), kernel density estimation, and Bayesian optimization to address both non-linear observation models and prediction models. Similarly, \citet{pmlr-v206-bickfordsmith23a} proposed the expected predictive information gain (EPIG), an acquisition function that measures information gain in the space of predictions rather than parameters. Taking a different approach, \citet{huang2024amortized} introduced a decision-aware framework with a transformer neural decision process that simultaneously generates experimental designs and infers decisions in a unified workflow. For causal discovery problems, \citet{tigas2022interventions} developed methods to optimize intervention timing for large nonlinear structural causal models. Collectively, these works represent a paradigm shift toward experimental designs that optimize directly for decision-relevant outcomes rather than intermediate parameter estimates. 

\textbf{Bayesian decision theory} Bayesian decision theory, which applies observed data to update posterior distributions for optimal decision-making, was formalized in \citet{chaloner1995bayesian}. Building on this foundation, \citet{pmlr-v15-lacoste_julien11a} developed a method that calibrates approximate inference techniques according to specific decision tasks using the Expectation-Maximization algorithm. For modern machine learning applications, \citet{NEURIPS2020_d3d94468} introduced a differentiable approach that balances accuracy against uncertainty calibration, enabling models to learn well-calibrated uncertainties while improving performance. Addressing computational efficiency challenges, \citet{gordon2018meta} developed a framework that uses few-shot learning to simplify posterior inference of task-specific parameters, eliminating the need for gradient-based optimization during testing. These advances have progressively made Bayesian decision-making more practical for complex problems with computational constraints.

\textbf{Robust decision-making} With the growing interest in goal-oriented BOED, robust decision-making has been studied in many application domains. For example, compartmental network-based approaches (e.g., SIQR model) are widely adopted in epidemic management. Two main control strategies dominate current research: optimal control to minimize infection rates \citep{lee2010optimal, hayhoe2021multitask, khanafer2014optimal, liu2020optimal, bock2018optimal} and spectral optimization for resource allocation \citep{hota2021closed, mai2018distributed, smith2023convex, preciado2014optimal, enyioha2015distributed}. A significant challenge with these approaches is their computational complexity, as many of the underlying problems are NP-complete or NP-hard~\citep{van2011decreasing}. 
In a parallel vein, PK models play a crucial role in optimizing drug dosing and improving patient outcomes by quantitatively linking individual variability to clinical efficacy and safety \citep{agema2025prospective, lai2022recent}. These models help guide dose selection and treatment personalization, especially under uncertainty in drug absorption, metabolism, and patient response \citep{zavvrelova2025population, norris2023large}.

\section{Structural Properties of GoBOED}
\label{app:structural}

This appendix proves \cref{thm:null-space-cancellation}: the GoBOED
pathwise gradient is insensitive to the null subspace $I-P^\ast$, so only
task-relevant parameter directions drive design updates.

\subsection{Setup}
\label{app:structural:setup}

$\btheta\in\mathbb{R}^d$ denotes the full parameter vector.
$P^\ast\in\mathbb{R}^{d\times d}$ is the orthogonal projection of rank
$k\le d$ onto the task-relevant subspace
\[
\mathcal{S}^\ast
= \mathrm{span}\{\nabla_{\btheta} c_j(g,\btheta) :
  g\in\mathcal{G},\,\btheta\in\mathbb{R}^d,\,j=1,\dots,m\},
\]
uniquely determined by the constraint functions $\{c_j\}$ alone.
Posterior samples are obtained by reparameterization,
$\btheta^{(i)} = h(y,\xi,\epsilon^{(i)})$ with $\epsilon^{(i)}\sim p(\epsilon)$.

\subsection{Assumptions}
\label{app:structural:assumptions}

\begin{assumption}[Task-Relevant Subspace]
\label{asm:subspace}
There exist a rank-$k$ orthogonal projection $P^\ast$ and functions
$\tilde c_j$ such that
\[
c_j(g,\btheta) = \tilde c_j(g, P^\ast\btheta)
\qquad\text{for all } g\in\mathcal{G},\,\btheta\in\mathbb{R}^d,\,j=1,\dots,m.
\]
Equivalently, $\nabla_{\btheta} c_j(g,\btheta)\in\mathrm{Im}(P^\ast)$
for all $g,\btheta,j$.
\end{assumption}

\begin{assumption}[Pathwise Gradient]
\label{asm:pathwise}
The optimal value $J^*(y,\xi)$ is differentiable with respect to the reparameterized samples $\theta^{(i)}$ a.e., and the optimal $g^*$ satisfies KKT conditions with bounded multipliers $\lambda_j^*$ a.e.
\end{assumption}

\subsection{Proof of Theorem~\ref{thm:null-space-cancellation}}
\label{app:structural:proof}

\begin{lemma}[Gradient Null-Space Cancellation]
\label{lem:null-space}
Under Assumption ~\ref{asm:subspace} and ~\ref{asm:pathwise},
\[
\frac{\partial J^\ast(y,\xi)}{\partial \btheta^{(i)}}\,(I-P^\ast) = 0,
\qquad
\frac{d J^\ast(y,\xi)}{d\xi}
= \sum_{i=1}^N
  \frac{\partial J^\ast}{\partial \btheta^{(i)}}\,P^\ast\,
  \frac{\partial \btheta^{(i)}}{\partial \xi}.
\]
\end{lemma}

\begin{proof}
\textit{Gradient via the envelope theorem.}
By Assumption ~\ref{asm:pathwise}, gradient signal flows through the reparameterized samples $\btheta^{(i)}$. The envelope theorem gives
\[
\frac{\partial J^\ast}{\partial \btheta^{(i)}}
= \sum_{j=1}^m \lambda_j^\ast\,\nabla_{\btheta} c_j(g^\ast, \btheta^{(i)}).
\]

\textit{Null-space cancellation.}
By Assumption ~\ref{asm:subspace}, $\nabla_{\btheta} c_j \in \mathrm{Im}(P^\ast)$, so $\nabla_{\btheta} c_j = P^\ast\nabla_{\btheta} c_j$. Right-multiplying by $(I-P^\ast)$ and using $P^\ast(I-P^\ast)=0$ gives 
\[
\frac{\partial J^\ast}{\partial \btheta^{(i)}}(I-P^\ast) = 0.
\]

\textit{Chain rule.}
Decomposing $\partial\btheta^{(i)}/\partial\xi = P^\ast\,\partial\btheta^{(i)}/\partial\xi + (I-P^\ast)\,\partial\btheta^{(i)}/\partial\xi$,
the null-space term vanishes and 
\[
\frac{d J^\ast}{d\xi}
= \sum_{i=1}^N
  \frac{\partial J^\ast}{\partial \btheta^{(i)}}\,P^\ast\,
  \frac{\partial \btheta^{(i)}}{\partial \xi}. \qedhere
\]
\end{proof}

\newpage

\end{document}

%% file: math_commands.tex
\usepackage{amsmath,amsfonts,bm}

\newcommand{\bbeta}{\boldsymbol{\beta}}

\newcommand{\bgamma}{\boldsymbol{\gamma}}
\newcommand{\btheta}{\boldsymbol{\theta}}

\newcommand{\vect}[1]{\bm{#1}}

\newcommand{\squishlist}{
 \begin{list}{$\bullet$}
 { \setlength{\itemsep}{0pt}
 \setlength{\parsep}{3pt}
 \setlength{\topsep}{3pt}
 \setlength{\partopsep}{0pt}
 \setlength{\leftmargin}{1.5em}
 \setlength{\labelwidth}{1em}
 \setlength{\labelsep}{0.5em} } }

\newcommand{\squishend}{
  \end{list}  }
  
\def\eqref#1{equation~\ref{#1}}

\def\1{\bm{1}}

\DeclareMathAlphabet{\mathsfit}{\encodingdefault}{\sfdefault}{m}{sl}
\SetMathAlphabet{\mathsfit}{bold}{\encodingdefault}{\sfdefault}{bx}{n}

